\documentclass{article}
\usepackage{subcaption}
\usepackage{wrapfig} % in preamble
\usepackage[nonatbib,preprint]{neurips_2026}
\usepackage{tcolorbox} % Put this in your preamble
\usepackage{amsthm}
\usepackage{xcolor} % in preamble
\definecolor{LightGray}{gray}{0.95}
\definecolor{HeaderBlue}{rgb}{0.3,0.3,0.6}

\usepackage{pifont,xcolor,booktabs,wrapfig}
\usepackage{wrapfig}

  \usepackage{amsmath}

\usepackage{amsmath}
\tcbset{
  colback=gray!5!white,
  colframe=black!30,
  fonttitle=\bfseries,
  title=LLM Prompt,
  coltitle=black,
  boxrule=0.5pt,
  arc=3pt,
  left=4pt,
  right=4pt,
  top=4pt,
  bottom=4pt
}
\newtheorem{theorem}{Theorem}

\newtheorem{remark}{Remark}
\newtheorem{assumption}{Assumption}
\newtheorem{lemma}{Lemma}

\usepackage{float}
\usepackage{tikz}
\usepackage{calc}
\usepackage{multirow}
\usepackage{pifont}

\usepackage{caption}
\DeclareCaptionFormat{myformat}{#1#2#3\hrulefill}
\usepackage{algorithm}
\usepackage{algorithmic} 
\usepackage{makecell} % Add this to your preamble if not already included
 
\tcbset{
  mycontribbox/.style={
  title=Contributions,
    colback=blue!5!white,
    colframe=blue!30!white,
    boxrule=0.5pt,
    arc=2mm,
    left=2mm,
    right=2mm,
    top=1mm,
    bottom=1mm,
    boxsep=3pt
  }
}

\usepackage{titlesec}
\newcommand{\AppendixBanner}{%
  \titleformat{\section}
    {\large\bfseries}
    {\thesection}{0.6em}{}
  \noindent\rule{\textwidth}{2pt}\vspace{2mm}
  \begin{center}
    {\LARGE\bfseries SUPPLEMENTARY INFORMATION}
  \end{center}
  \noindent\rule{\textwidth}{1pt}\vspace{2mm}
}

\definecolor{okgreen}{RGB}{0,128,0}
\definecolor{badred}{RGB}{180,0,0}
\definecolor{amber}{RGB}{184,134,11}

\tcbset{
  myexamplebox/.style={
    notitle,
    colback=blue!7!white,
    colframe=blue!30!white,
    boxrule=0.5pt,
    arc=2mm,
    left=2mm,
    right=2mm,
    top=1mm,
    bottom=1mm,
    boxsep=3pt
  }
}

\usepackage[utf8]{inputenc} % allow utf-8 input

\usepackage[backend=biber, style=ieee, doi=false,url=false]{biblatex}
\usepackage{amsmath}

\usepackage[T1]{fontenc}    % use 8-bit T1 fonts
\usepackage{hyperref}       % hyperlinks
\hypersetup{
    colorlinks=true,
    linkcolor=orange,
    filecolor=magenta,      
    urlcolor=cyan,
    citecolor=blue,
    pdftitle={eri},
    pdfpagemode=FullScreen,
    }
\usepackage{url}            % simple URL typesetting
\usepackage{booktabs}       % professional-quality tables
\usepackage{amsfonts} 
\usepackage{xcolor}
\usepackage{nicefrac}       % compact symbols for 1/2, etc.
\usepackage{microtype}      % microtypography
\usepackage{xcolor}         % colors
\usepackage{booktabs}

\title{Fast Generalization after Interpolation via Critically Damped Momentum Optimization}

\author{%
  Luca Muscarnera  \\
  University of Cambridge 
  \And
  Silas Ruhrberg Estévez \\
  University of Cambridge 
  \AND
  Yuanzhang Xiao \\
  University of Hawaii at Manoa 
  \And 
  Mihaela Van der Schaar
  \\
  University of Cambridge 
}

\begin{document}

\maketitle

\begin{abstract}
A central problem in machine learning is that models can achieve near-perfect training performance while generalizing substantially less well to unseen examples. This gap is especially acute in high-dimensional, low-sample regimes, where many interpolating solutions exist and optimization must implicitly select among minima with different generalization properties. Following recent theoretical advances on optimization dynamics near the interpolation threshold, we note that the two-regime structure of risk minimization, with loss minimization followed by complexity minimization, motivates a biphasic optimization schedule. We thus theoretically demonstrate that \texttt{GROKtimizer}, a biphasic strategy that combines rapid convergence to interpolation with Critically Damped Momentum (CDM)-based post-interpolation norm minimization, offers a natural solution for selecting low-norm interpolating solutions. Under a local quadratic model of the post-interpolation basin, \texttt{GROKtimizer} provides a quadratic speedup over classical gradient descent, with provable optimality among first-order optimizers. To showcase the applicability of our method, we evaluate \texttt{GROKtimizer} on several synthetic benchmarks common in the classical grokking literature and on various real-world datasets. Finally, we reconcile our findings with the flat-minima hypothesis, highlighting the importance of post-interpolation dynamics in the construction of high-quality, generalizing models.
\end{abstract}

\section{Introduction}

Training machine learning models that generalize reliably remains difficult, especially in high-dimensional regimes where data are sparse, noisy, or expensive to obtain \cite{zhang2017understanding}. Such regimes arise in genomics and single-cell biology, where each sample may contain thousands of molecular measurements but few labels \cite{FeldnerBusztin2023}; in medicine, where clinically meaningful outcomes are often rare \cite{Rajkomar2018}; in finance, where signals are weak and non-stationary \cite{CHINCO2018}; and in materials science or drug discovery, where experiments are costly \cite{Xu2023, Dou2023}. Across these settings, a model may fit the available data without learning a predictor that extrapolates robustly to unseen examples.

This difficulty is closely tied to the geometry of modern loss landscapes \cite{Garipov2018}. When the number of effective degrees of freedom is large relative to the amount of data, there may be many parameter configurations that explain the training set equally well \cite{draxler2018essentially}. Interpolation therefore does not identify a unique model \cite{Belkin2019}. Instead, optimization must implicitly select among many solutions that achieve similar or identical training loss but differ substantially in their complexity and test performance \cite{Soudry2018}. The challenge is therefore not only to reach low training error, but to reach an interpolating solution that generalises.

Regularization methods were developed precisely to guide this selection process. Weight decay, norm penalties, early stopping, and related methods bias optimization away from arbitrary high-complexity interpolators and toward simpler solutions \cite{Krogh1991}. This reflects a classical Occam-style principle: among models that explain the observed data, prefer the one that is less likely to fit the data by coincidence  \cite{Rasmussen2005} \cite{balasubramanian1997statistical}. In the optimization dynamics studied in this work, this principle appears through the parameter norm: among interpolating solutions, lower-norm solutions provide a tractable notion of lower complexity \cite{hastie2022surprises}.

Grokking provides a striking example of this separation between fitting and generalization \cite{power2022grokking, kumar2023grokking}. In grokking, a model first reaches near-perfect training performance, often appearing to have memorized the data, and only much later undergoes a delayed transition to strong test performance. This phenomenon suggests that training can proceed in two qualitatively different phases: an initial phase of empirical risk minimization, followed by a slower post-interpolation phase in which optimization continues to move through the space of interpolating solutions toward lower-complexity predictors. However, the practical value of this phenomenon is limited by its timescale. In classical grokking benchmarks, the number of epochs required for generalization can be orders of magnitude larger than the number required to fit the training set \cite{power2022grokking,liu2022towards}. As a result, grokking is often studied on synthetic or highly simplified tasks where extremely long training runs are feasible \cite{jeffares2025not}.

In this work, we address this gap by studying the post-interpolation phase as an optimization problem in its own right. We argue that, after interpolation, the relevant objective is no longer primarily to reduce training loss, but to select a low-complexity solution within or near the set of interpolating parameters. Under a local quadratic approximation of the regularized loss, this phase can be described as damped dynamics along the flat directions of the loss landscape. This perspective relates minimum-norm solution selection to the geometry of flat regions in the loss landscape and turns the choice of post-interpolation optimization dynamics into a damping problem \cite{Hochreiter1997}. 

\paragraph{Contributions.}
\begin{enumerate}
    \item We characterize the post-interpolation regime as a local quadratic dynamical system, linking delayed generalization to damped oscillator behavior along flat directions of the loss landscape.
    \item We show that critically damped momentum accelerates convergence toward low-norm interpolating solutions in this regime, yielding a principled momentum schedule for post-interpolation training.
    \item We propose \texttt{GROKtimizer}, a biphasic optimization strategy that first drives models to interpolation and then switches to critically damped, weight-decayed momentum dynamics.
    \item We validate our approach empirically on classical grokking benchmarks and real-world high-dimensional datasets.
\end{enumerate}

\section{Theoretical Considerations}
\textbf{Parameter norm as a proxy for complexity.}
The relationship between parameter geometry and generalization has been widely studied in modern machine learning \cite{hochreiter1997flat, jacot2018neural}. A particularly important case is overparameterized regression \cite{bartlett2020benign}, where infinitely many solutions can interpolate the training data. In this setting, the minimum-norm interpolating solution plays a central role: although all interpolating solutions achieve identical training error, they can differ substantially in their behavior on unseen data.

A surprisingly simple model of this idea can be explicitly obtained by studying overparameterized isotropic regression, where interpolation leaves a null-space of solutions with identical training error but different test behavior.

\begin{lemma}[Generalization Asymptotics for Overparameterized isotropic regression] 
Let $X = \{\boldsymbol{x}_1,...,\boldsymbol{x}_n\} \subset \mathbb R^p$ be a dataset such that $\boldsymbol{x}_i \sim \mathcal N(\mathbf 0, \mathbf I_p) \ \forall i = 1,...,n$, and define the data matrix $\mathbf X = \sum_{i=1}^n \boldsymbol{e}_i \boldsymbol{x}_i^\top$. Assume $p > n$ and consider a noisy target vector \(
    \boldsymbol{y} = \mathbf X \boldsymbol{w}_* + \sigma \boldsymbol{\varepsilon} \)
where $\boldsymbol{\varepsilon} \sim \mathcal N(\mathbf 0, \mathbf I_n)$, $\sigma \in \mathbb R_+$ and $\boldsymbol{w}_* \in \mathbb R^p$. Consider the minimum-norm interpolating solution $
    \widehat{\boldsymbol{w}}_*$ and define the family of interpolating solutions as \(
    \widehat{\boldsymbol{w}}(\boldsymbol{z}) = 
    \widehat{\boldsymbol{w}}_* + \Pi_{\mathsf{ker}(\mathbf X)}\boldsymbol{z} \)
and assume $\boldsymbol{z}$ independently and uniformly distributed over $r \mathbb S^{p-1}$, for some $r> 0$.
Then in the limit $n,p \to \infty$ with $n/p = \gamma \in (0,1)$, the following asymptotic expression for the generalization error on an unseen sample $\boldsymbol{x} \sim \mathcal N(\mathbf 0 ,\mathbf I)$ holds:
\begin{equation}
    \mathbb E (\boldsymbol{x}^\top \widehat{\boldsymbol{w}}(\boldsymbol{z}) - \boldsymbol{x}^\top \boldsymbol{w}_* )^2 \sim 
    (1 - \gamma) \| \boldsymbol{w}_* \|^2 + \mathbb E \| \widehat{\boldsymbol{w}}(\boldsymbol{z}) - \widehat{\boldsymbol{w}}_*\|^2_2 + \sigma^2 \frac{\gamma}{1 - \gamma}
\end{equation}
\begin{proof}
This asymptotic result is a corollary of Theorem 1 from \cite{hastie2022surprises}. A detailed derivation is provided in the Appendix.
\end{proof}
\end{lemma}

This result shows that, even among exact interpolators, displacement in the null space of the data matrix incurs an explicit test-error penalty.

This behavior can also be observed empirically in neural networks, although with a different polynomial dependence on $\| \hat {\boldsymbol{w}}  - \hat{\boldsymbol{w}}_*\|_2$. In Figure \ref{fig:scaling}, we vary the initialization scale to obtain interpolating solutions with different weight norms and evaluate their test cross-entropy on MNIST and CIFAR-10. In both cases, test loss follows an approximate power law $L \propto \| \boldsymbol{w} \|^\alpha$ with $\alpha \sim 3$. Although the exponent differs from the isotropic Gaussian setting, likely due to differences in loss, architecture, and data geometry, the qualitative polynomial dependence on parameter norm persists. This supports the role of parameter norm as a practical proxy for model complexity. Lower-norm interpolating solutions can therefore be viewed as empirically favored by an Occam-style principle, motivating the need for efficient methods that search for minimum-norm models.

%\textbf{Scaling laws for generalization error and weight norm}

%\textbf{Critically dampened momentum optimization}
\begin{wrapfigure}{r}{0.35\textwidth}
\vspace{-12pt}
  \centering
  \includegraphics[
    width=\linewidth,
    trim={0cm 0cm 1.7cm 0.62cm},
    clip
  ]{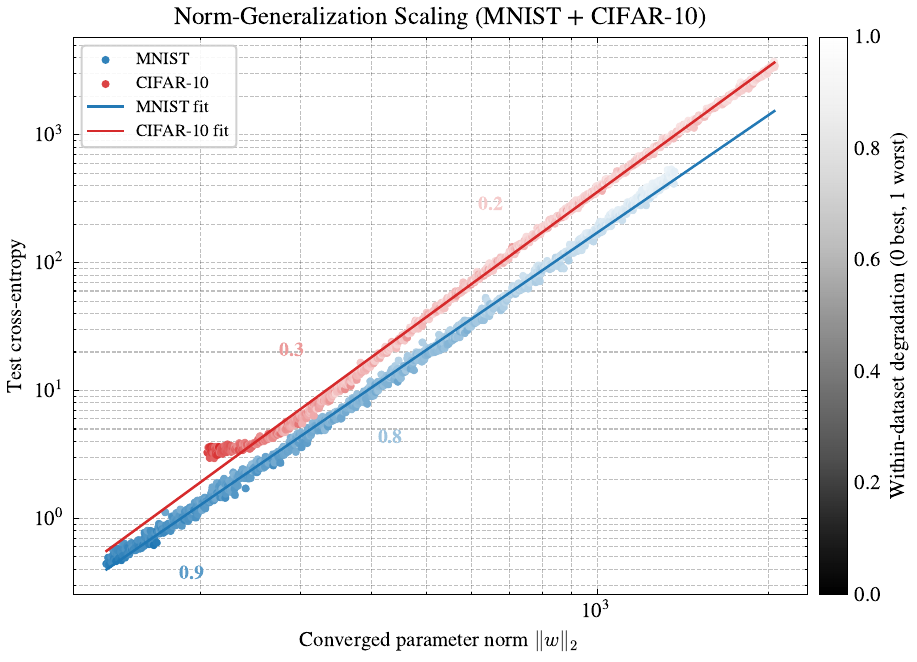}
  \caption{\textbf{Loss-Norm scaling laws.}
  Test accuracy (indicated by point brightness and inline labels) reveals shifts in the classification boundary across different norms.}
  \label{fig:scaling}
\vspace{-10pt}
\end{wrapfigure}
\paragraph{Qualitative and quantitative differences in fitting and complexity minimization}
Recent literature \cite{musat2025geometry,boursier2025theoretical} identifies the onset of grokking as a dynamical transition between a regime of \emph{data memorization} and \emph{norm (complexity) minimization}.
In particular, \cite{boursier2025theoretical} shows the existence of two distinct phases. In the first part of the optimization, the dynamics is dominated by the unregularized gradient flow induced solely by the training data. In this regime, the data memorization, empirical risk is minimized and the optimizer appears to be \emph{resilient} to regularization. 
In this part of the empirical risk minimization, the characteristic timescale strongly depends on data, model architecture and the number of parameters; for linear models, this can be proven to be controlled by an exponential decay parametrized by the smallest nonzero eigenvalue of the uncentered covariance matrix. For Neural Network such analysis is more complex, but studies on lazy training dynamics show analogous results on the Neural Tangent Kernel Gram matrix \cite{lee2019wide}, suggesting a possible connection between the two interpolation dynamics . It is, however, empirically (and theoretically) well-known that larger Neural Network based models are able to interpolate data faster than their smaller counterparts \cite{allen2019convergence}.

The second phase, which is reached only in models with enough capacity to memorize the data, shows a qualitative and quantitative different behavior. The dynamics are shaped mainly by the the effect of regularization; the presence of a complexity penalty in the loss function (such as the squared norm of parameters) steers the optimizer toward solutions of minimal complexity within the basin of attraction. From a quantitative description of such dynamics, the timescale of the second phase for $L^2$ regularization is provably  $\mathcal \tau = \lambda^{-1}$.  We face a fundamental discrepancy between these two regimes of the optimization. 

\begin{remark} The convergence rate to the interpolation threshold can be accelerated by overparametrization, while the slow dynamics of norm minimization is driven by the regularization magnitude.
\end{remark}

\begin{figure}[!htb]
    \centering
    \begin{subfigure}{0.48\linewidth}
    \centering
    \includegraphics[width=1.\linewidth]{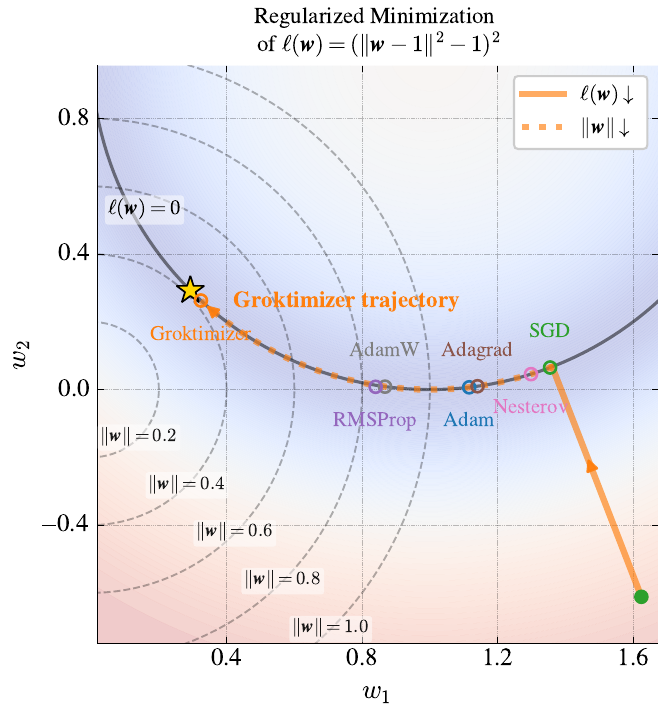}
    \caption{Fixed-budget dynamics on a Modular Addition benchmark. \texttt{GROKtimizer} reaches interpolation without regularization, then applies CDM to minimize norm along the zero-loss manifold, converging faster than standard optimizers.}
    \label{fig:toy_dynamics}
    \end{subfigure}
    \hfill
    \begin{subfigure}{0.48\linewidth}
    \centering
    \includegraphics[width=1.\linewidth]{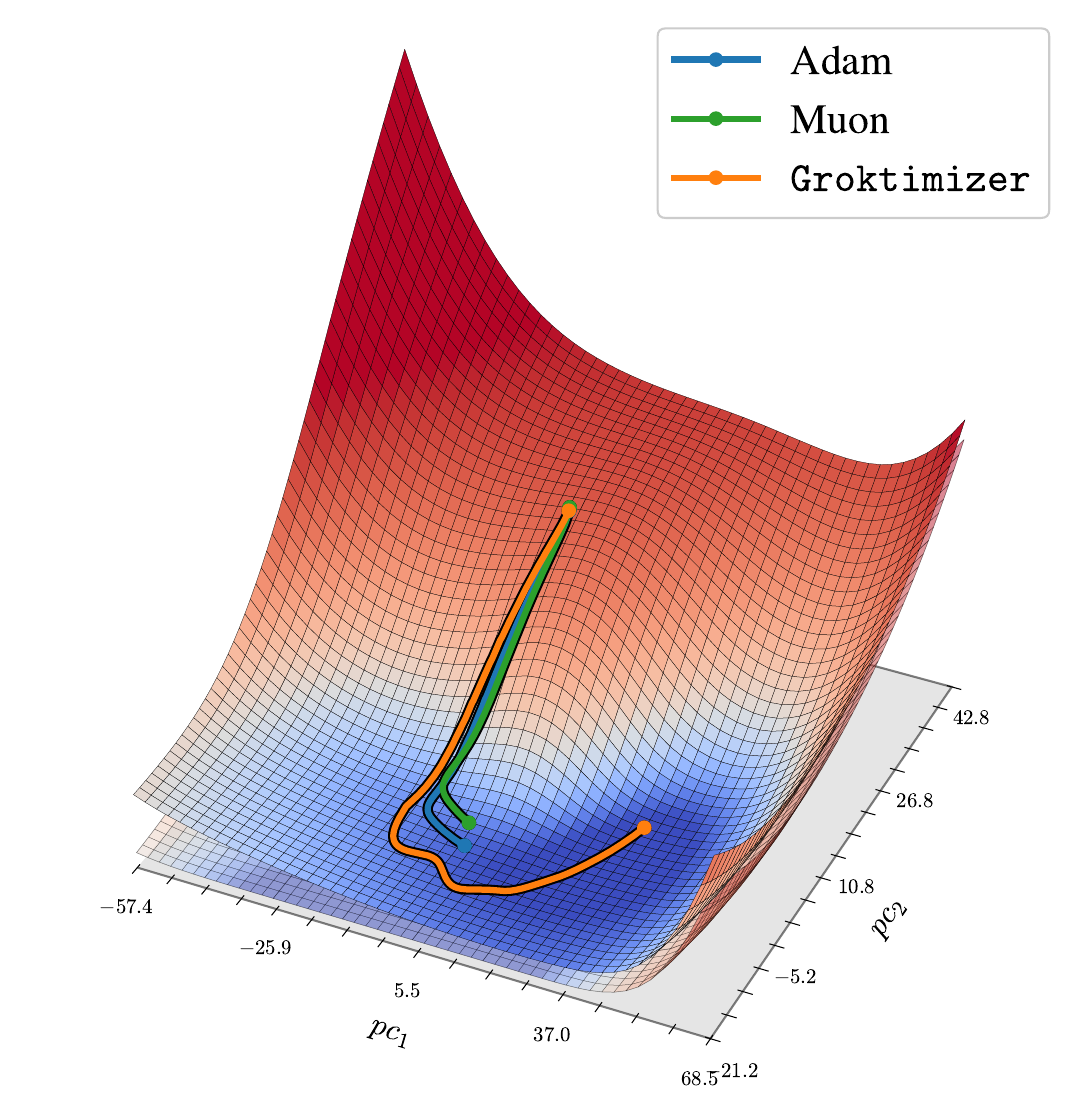}
    \caption{Reduced loss-landscape visualization for Modular Addition. Optimizer trajectories are projected onto a two-dimensional subspace; \texttt{GROKtimizer} uses the regularized landscape to move efficiently through the flat interpolation region toward lower-norm solutions.}
    \label{fig:loss_landscape}
    \end{subfigure}
    \caption{\texttt{GROKtimizer} dynamics before and after interpolation.}
    \label{fig:groktimizer_dynamics}
\end{figure}

\paragraph{Characterizing the after-interpolation behavior of slow dynamics}
Given the different features of the dynamics after interpolation, we now study in detail its nature in the context of neural networks. 

\begin{assumption} \label{ass:quad} Consider a machine learning model $f$ characterized by a set of parameters $\boldsymbol{w}^\star$, obtained by training the network until interpolation of the training set. Then the $L^2$ regularized loss function can be approximated as
\begin{equation}
    \mathcal L_{reg}(\boldsymbol{w}) \approx \mathcal L_{unreg}(\boldsymbol{w}^\star) + \frac{1}{2}(\boldsymbol{w} - \boldsymbol{w}^\star)^\top \mathbf H (\boldsymbol{w} - \boldsymbol{w}^\star) + \frac{1}{2} \lambda \| \boldsymbol{w} \|_2^2
\end{equation}
with $\mathcal L(\boldsymbol{w}^\star)$ being a constant, $\mathbf H$ being the Hessian matrix evaluated in $\boldsymbol{w}^\star$, and $\mathcal L(\boldsymbol{w}^\star) = 0$ in case of Mean Squared Error Loss.
\end{assumption}

This assumption is motivated by the proximity to a local minimum; when the optimizer is captured in an attraction basin, the geometry of the surrounding loss landscape is entirely characterized by the Hessian matrix. Furthermore, under the fast-slow dynamics formalism, we can expect the distance $d^\star = \| \boldsymbol{w} - \boldsymbol{w}^\star \|$ to grow slowly in time thus justifying the Taylor-like approximation.
A clear advantage of this model is its perfect equivalence to the linear case; since in the attraction basin the function is approximately quadratic, we can adopt the same mathematical framework to analyze the dynamics of the optimizer.

\begin{remark}[Manifold of Optima]
     Classical Deep Learning literature \cite{garipov2018loss,benton2021loss}, highlights how in the loss landscape of neural networks optima are connected through a manifold in the parameter space. While our discussion focuses on a linear subspace (the null space of the Hessian matrix) we argue that our model can be reconciled when one considers such null space as the tangent space of such a manifold. Our assumption coincides, in fact, with the idea of small variations  of the tangent space across the various solutions \cite{draxler2018essentially}.
\end{remark}

%As shown in \cite{gur2018gradient}, in fact, gradient descent dynamics in deep networks rapidly concentrate in a low-dimensional subspace associated with the top curvature directions of the loss landscape, which becomes relatively stable during training. This motivates our assumption, corroborating our idea that the characteristic timescale of weights evolution is larger than the one of the Hessian and validating the usage of this assumption in the construction of our method.

Under this model, we can provide the central result of the paper.
\begin{theorem}[Critically Damped Momentum Quadratically Accelerates Convergence]
    Consider a machine learning model $f$ trained to interpolation on any unregularized loss function $\mathcal L$, described by a set of parameters $\boldsymbol{w}^\star$, and assume Assumption \ref{ass:quad} holds. Then, optimizing the regularized loss function $\mathcal L_{reg} = \mathcal L + \frac{1}{2} \lambda \| \cdot \|^2$ initialized on $\boldsymbol{w}^\star$ and optimized through a momentum method tuned to critical damping with $\beta = 1 - 2 \sqrt{\lambda \eta}$  reaches the minimum norm solution of the quadratic approximation quadratically faster than Gradient Descent, and in particular
    \begin{equation}
        E \sim \sqrt{E_{GD}}
    \end{equation}
\end{theorem}
\begin{proof}
Let us define the quantity
\begin{equation}
    \gamma^u  := \frac{1}{\sqrt \eta} ( 1 - \beta ),
\end{equation}
then by a Gradient Flow type of argument we can derive the continuous characterization of the optimizer dynamics in the $\sqrt{\eta} \to 0$ limit
\begin{equation}
    \ddot{\boldsymbol{w}} + \gamma^u \dot{\boldsymbol{w}} +   \nabla \mathcal L(\boldsymbol{w}) = \boldsymbol{0}
\end{equation}
Under the assumption of local convexity, we can express the dynamics projected on the coordinate system defined by the eigenbasis $\{(\lambda_i, \boldsymbol{v}_i)\}_{i=1}^N$ induced by the Hessian matrix $\mathbf H$ as 
\begin{equation}
   \boldsymbol{v}_i^\top \ddot {\boldsymbol{w}} +
   \gamma^u \boldsymbol{v}_i^\top \dot {\boldsymbol{w}} +\boldsymbol{v}_i^\top \nabla \mathcal L(\boldsymbol{w}) = 0 \ \ \ \forall i 
\end{equation}
which, under explicit representation of the gradient $\nabla \mathcal L$ and application of Assumption 1, admits the form
\begin{equation}
\begin{split}
   \partial_t^2 \left \{ \boldsymbol{v}_i^\top  {\boldsymbol{w}} \right\} +
    \gamma^u \partial_t \left\{ \boldsymbol{v}_i^\top {\boldsymbol{w}} \right\}
   +  (\lambda_i + \lambda) \boldsymbol{v}_i^\top \boldsymbol{w} -
   \lambda_i \boldsymbol{v}_i^\top \boldsymbol{w}^*
   & = 
0
\end{split}
\end{equation}
and defining $\tilde x_i = \boldsymbol{v}_i^\top \boldsymbol{x}$ we obtain the system of decoupled second order ordinary differential equations

\begin{equation}
    \partial_t^2 \tilde w_i + \gamma^u \partial_t \tilde w_i + 
    (\lambda_i + \lambda) \tilde w_i - \lambda_i \tilde w^*_i = 0 \ \ \ \forall i
\end{equation}

which represents a system of independent harmonic oscillators. Components of the parameter vector associated to zero eigenvalues of the Hessian are \emph{critically damped} when $\gamma^u = 2 \sqrt{ \lambda}$
and in this case their evolution is described as 
\begin{equation}
    \tilde w^{CD}_i(t) = 
    \tilde w^{CD}_i(0) \left(1 + {\gamma^u}{t} \right) e^{ - \frac{\gamma^u}{2}t}  \ \ \ \forall i : \lambda_i =0 
\end{equation}
which follows the scaling
\begin{equation}
     \tilde w_i^{CD}(t) \sim  e^{ - \sqrt{ \lambda } t }  \ \ \ \forall i : \lambda_i = 0
\end{equation}
against the $\sim e^{- 2 \lambda t}$ of classical gradient descent. Substituting $\gamma^u$ into $\beta$ we obtain the momentum rate 
\begin{equation}
    \beta = 1 - 2 \sqrt{ \lambda \eta}
\end{equation}
that enables the critical damping of the optimization. Since the characteristic time scale $\tau$ of an optimization method and number of epochs $E$ is described by the approximation
\begin{equation}
    E \Delta t \sim \tau 
\end{equation}
where $\Delta t$ is the discretization step of the continuous approximation of the optimizer, substituting $\tau_{CDM} = \lambda^{-1/2}$ and $\tau_{GD} = \lambda^{-1}$, and substituting $\Delta t_{CDM} = \sqrt{\eta}$ and $\Delta t_{GD} = \eta$ we obtain
\begin{equation}
    E_{GD} \sim \lambda^{-1} \eta^{-1}, \ \ \
    E_{CDM} \sim \lambda^{-1/2} \eta^{- 1/2} \implies E_{CDM} \sim \sqrt{E_{GD}}
\end{equation}
A detailed proof and the proof of small $\mathcal O(\lambda \lambda^{-1}_i)$ fluctuations for nondegenerate components are provided in the Appendix.
\end{proof}

Remarkably, in the context of first-order black-box optimization of smooth convex functions, Nesterov \cite{Nesterov1983AMF,nesterov2013introductory} established that any algorithm accessing the objective and its gradients cannot, in general, achieve a convergence rate faster than $O(1/k^2)$ in terms of function value error after $k$ iterations. Hence, within the class of first-order methods, our proposed quadratic improvement over standard gradient descent is asymptotically information-theoretically optimal.

\noindent
\paragraph{Optimal Learning Rate for CDM.}
A fundamental question behind the introduction of our biphasic framework is whether an optimal timestep exists in the discretization of the critically damped dynamics. The following result formalizes this idea, providing a stable maximal learning rate that can be adopted under Assumption \ref{ass:quad}.

\begin{theorem}
Consider a basin of attraction characterized by a Hessian matrix $\mathbf H$ whose largest eigenvalue is $\lambda_{\max} > 0$. Then the maximum admissable learning rate is given by the relationship
    \begin{equation}
        \eta_{\max} = \frac{
4 \left( 
\sqrt{\lambda_{\max} + 2 \lambda} - \sqrt{\lambda}
\right)^2
    }{(\lambda_{\max} + \lambda)^2}
    \sim 4 \lambda_{\max}^{- 1}
    \end{equation}
\end{theorem}
\begin{proof}
This result is a direct corollary of the stability threshold $\eta \le 2(1 + \beta) \lambda_{\max}^{-1}$ in heavy ball optimizers, and the identity is obtained solving a two-dimensional nonlinear system that encodes the mutual dependency of $\beta$ and $\eta$ in the CDM framework. A detailed proof is provided in the Appendix.
\end{proof}

\begin{minipage}{0.5\textwidth}
We study this result numerically, and it is presented in Figure \ref{fig:wishart}. By initializing a rank deficient Wishart random matrix $\mathbf H$, we define a $L^2$ regularized random loss landscape based on the quadratic form $\boldsymbol{w} \mapsto (\boldsymbol{w} - \boldsymbol{w}^*)^\top \mathbf H (\boldsymbol{w} - \boldsymbol{w}^*)$, and initialize $n$ different CDM optimizers in a point of the nullspace of $\mathbf H$, each with a different learning rate $\eta$. To enforce the minimum norm condition of the planted solution, we sample a vector $\widetilde{\boldsymbol{w}*}$ and define $\boldsymbol{w}^* := \mathbf H^+ \mathbf H \widetilde{\boldsymbol{w}*}$ and denote it as the \emph{Pseudo-Inverse } solution; we know in fact that this specific expression solves the Ridgeless-Regression problem and thus provides the minimum norm optimum for this specific random energy landscape. Numerically, we observe that the number of iterations spent seamlessly inside the ball of radius $\sqrt{c}$ centered in the Pseudo-Inverse solution grows monotonically with $\eta$ until the reach of $\eta_{max}$, which induces a sharp phase transition between the convergence regime and the unstable regime.
\end{minipage}
\hfill
\begin{minipage}{0.48\textwidth}
    \centering
    \includegraphics[width=\linewidth]{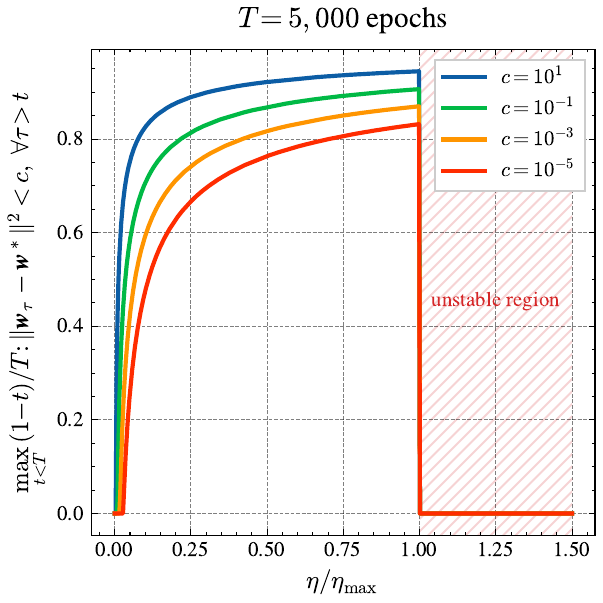}
    \captionof{figure}{Optimal learning lies at the \emph{edge of criticality} between convergence and instability.}
    \label{fig:wishart}
\end{minipage}

From the perspective of practical application, $\eta$ of the CDM phase can be treated as a hyperparameter of our method. Alternatively, a practitioner may be interested in an efficient way to estimate $\lambda_{\max}$, compatibly with the computational intractability of the Hessian matrix in neural networks, in order to approximate the optimal learning rate in an online fashion prior to the norm minimization phase. An efficient solution for this problem is to employ a Power-Iteration based estimate on top of the Perlmutter Trick, which allows to efficiently compute the action of the Hessian against a vector in a Matrix-Free fashion. The classical power method $\boldsymbol{v}_{k+1} \propto \mathbf H \boldsymbol{v}_{k}, \lambda_{k} = \boldsymbol{v}_{k}^\top \mathbf H \boldsymbol{v}_k $ has however the disadvantage of negative bias; since the sequence of eigenvalues estimates are bounded from above by the largest eigenvalue, $\hat \eta_{\max}$ will be an overestimate of the critical learning rate and thus part of the unstable region. To remove this systematic bias and accelerate convergence we adopt Aitken correction of the power method, which also accelerates the convergence rate of the estimate, making it suitable for a few-step ($\sim 10,20$ in our experiments) estimation phase prior of CDM. Remarkably, the application of Aitken correction has a negligible computational cost, since the number of Hessian-vector actions is one per iteration as in classical Power Method and the correction requires only the additional storage of the previous three eigenvalues, which are scalars. Relaxing our working assumption, moreover, a corrective multiplicative coefficient $\alpha \in (0,1)$ may be applied to the estimate, to make the method robust to possible changes of curvature.

\section{Related Work}

\paragraph{Optimization dynamics and momentum methods.}
Momentum-based methods are among the most widely used tools for accelerating first-order optimization. Classical heavy-ball momentum \cite{Polyak1964} and Nesterov acceleration \cite{Nesterov1983AMF} show that optimizer dynamics can substantially change convergence rates, while later work emphasized the practical importance of momentum, initialization, and conditioning in deep neural network training \cite{pmlr-v28-sutskever13,}. Adaptive optimizers such as Adam \cite{kingma2015adam}, AdamW \cite{Loshchilov2017DecoupledWD}, and related variants have become standard in deep learning, although their behavior is governed by generic stability and adaptivity considerations rather than the geometry of a specific training regime. The specific regime of critical damping, the boundary between oscillatory and overdamped decay in a second-order dynamical system, is well-studied in control theory and classical mechanics as the condition that minimizes settling time without oscillation \cite{Ogata2001}. Our work connects this classical notion directly to post-interpolation optimization: under a local quadratic model of the loss basin, critically damped momentum yields the fastest non-oscillatory decay of the norm-minimizing dynamics along flat directions.

\paragraph{Implicit bias and post-interpolation solution selection.}
A large body of work studies how optimization selects among multiple interpolating solutions in overparameterized models. In linear models and related regimes, gradient-based methods are known to exhibit implicit biases toward particular interpolating predictors, such as minimum-norm or max-margin solutions \cite{soudry2018the,gunasekar2018implicit}. Similar ideas have been connected to implicit regularization and generalization in deep networks \cite{neyshabur2017implicit}, and the statistical properties of minimum-norm interpolation have been characterized in high-dimensional asymptotics \cite{hastie2022surprises}. Recent analyses of grokking and post-interpolation training identify a more explicit two-phase structure: rapid empirical risk minimization followed by slower regularization-driven norm minimization on or near the interpolation manifold \cite{boursier2025theoretical,musat2025geometry}. Our work builds directly on this two-phase view, providing an optimizer explicitly designed for the second phase.

\paragraph{Grokking-inspired optimizer design.}
Grokking provides a useful model system for studying the separation between fitting and generalization. Since its original observation in small algorithmic tasks \cite{power2022grokking}, grokking has been interpreted through several complementary mechanisms, including transitions from lazy to rich feature learning \cite{kumar2023grokking,kumar2024grokking}, circuit efficiency and implicit bias \cite{varma2023explaining}, oscillatory weight-norm dynamics \cite{thilak2022slingshot}, and the emergence of structured representations such as Fourier features \cite{nanda2023fourier}. Taken together, these accounts provide valuable insight into the mechanisms underlying delayed generalization, but remain largely descriptive: they characterize when and why grokking occurs without offering a principled optimization strategy for shortening the delay.

\paragraph{Accelerating post-interpolation generalization.}
Several recent methods aim to accelerate delayed generalization. Grokfast amplifies slow gradient components \cite{lee2024grokfast}, NeuralGrok learns gradient transformations through auxiliary networks \cite{zhou2025neuralgrok}, Muon applies approximate second-order spectral updates \cite{tveit2025muon}. Other approaches modify the architecture or data distribution, for example through Kolmogorov--Arnold representations or data augmentation strategies \cite{park2025acceleration,abramov2025grokking}. These methods demonstrate that delayed generalization can be accelerated, but typically rely on gradient filtering, learned transformations, architectural choices, or task-specific data interventions. In contrast, \texttt{GROKtimizer} is derived from a local quadratic model of the post-interpolation basin and yields a simple biphasic optimizer: rapidly reach interpolation, then switch to critically damped, weight-decayed momentum dynamics to accelerate movement toward low-norm interpolating solutions.

\section{Experimental results}

We now empirically investigate whether an optimizer inspired by the two-phase structure of grokking can accelerate post-interpolation generalization. Our experiments are designed to answer four questions:
\begin{enumerate}
    \item Does \texttt{GROKtimizer} reduce delayed generalization on synthetic grokking benchmarks?
    \item Do the same dynamics transfer to noisy real-world datasets?
    \item In data-scarce regimes, does \texttt{GROKtimizer} improve sample efficiency?
    \item Can the same biphasic principle be scaled to modern transformer architectures for language-model pretraining?
\end{enumerate}

\paragraph{Evaluation protocol.}
Across all experiments, we deliberately operate in data-scarce regimes. Such regimes arise naturally in domains such as genomics, multi-omics, and clinical prediction, where the number of measured variables can be large relative to the number of labeled examples. For image and language tasks, data scarcity is less intrinsic, but we use controlled low-data settings to isolate the effect of optimizer dynamics rather than dataset scale. All optimizers are evaluated using the same model architecture, data split, batch size, training budget, and non-optimizer hyperparameters. We compare against Adam, AdamW, Muon, and SGD. The distinguishing feature of \texttt{GROKtimizer} is its biphasic schedule: in the first phase, weight decay is set to zero to prioritize rapid interpolation; after interpolation, the optimizer switches to a critically damped, weight-decayed momentum phase to minimize parameter norm within the interpolating region.

\begin{table}[t]
\centering

\caption{\textbf{Synthetic benchmark results.}
Gaussian reports validation loss, where lower is better; all other columns report test accuracy in percent, where higher is better.}
\label{tab:main_benchmark_results}
\scriptsize
\setlength{\tabcolsep}{4.2pt}
\renewcommand{\arraystretch}{1.15}
\resizebox{\textwidth}{!}{%
\begin{tabular}{lcccccc}
\toprule
\textbf{Method}
& \makecell{\textbf{Gaussian}\\\textbf{Val. Loss} $\downarrow$}
& \makecell{\textbf{Binary}\\\textbf{Addition} $\uparrow$}
& \makecell{\textbf{Modular}\\\textbf{Addition} $\uparrow$}
& \makecell{\textbf{RF Teacher}\\\textbf{Linear} $\uparrow$}
& \makecell{\textbf{Sparse}\\\textbf{Parity} $\uparrow$}
& \makecell{\textbf{Two-Subspace}\\\textbf{Linear} $\uparrow$} \\
\midrule
Adam
& $121.7 \pm 12.7$
& $82.9 \pm 0.9$
& $70.1 \pm 16.8$
& $78.6 \pm 5.5$
& $60.5 \pm 2.2$
& $79.5 \pm 1.0$ \\

AdamW
& $151.4 \pm 17.0$
& $36.6 \pm 2.2$
& $0.0 \pm 0.0$
& $55.5 \pm 2.5$
& $69.6 \pm 3.8$
& $66.3 \pm 1.7$ \\

Muon
& $94.2 \pm 9.5$
& $0.31 \pm 0.07$
& $0.27 \pm 0.12$
& $50.7 \pm 1.8$
& $52.1 \pm 0.8$
& $59.2 \pm 1.3$ \\

SGD
& $94.2 \pm 9.5$
& $0.04 \pm 0.04$
& $1.83 \pm 0.20$
& $56.1 \pm 1.9$
& $50.3 \pm 0.9$
& $58.3 \pm 1.3$ \\

\midrule
\texttt{GROKtimizer}
& $\mathbf{(4.3 \pm 0.7)\times 10^{-5}}$
& $\mathbf{89.5 \pm 1.0}$
& $\mathbf{95.1 \pm 1.4}$
& $\mathbf{84.3 \pm 6.5}$
& $\mathbf{89.1 \pm 3.2}$
& $\mathbf{79.7 \pm 1.0}$ \\

\bottomrule
\end{tabular}%
}
\vspace{-20pt}
\end{table}

\paragraph{Synthetic grokking benchmarks.} We first evaluate \texttt{GROKtimizer} on synthetic benchmarks designed to expose a separation between fitting and generalization. These tasks are useful testbeds for our hypothesis because they admit rapid training-set interpolation while requiring additional optimization dynamics to recover a low-complexity solution that generalizes. Table~\ref{tab:main_benchmark_results} reports the main results over $n=10$ random seeds. \texttt{GROKtimizer} achieves the best mean performance on all six benchmarks. On the Gaussian task, it reduces validation loss by several orders of magnitude relative to the strongest baseline, indicating that the post-interpolation phase effectively drives the model toward a lower-complexity interpolating solution.

\begin{wraptable}{r}{0.5\textwidth}
\vspace{-10pt}
\centering
\caption{\textbf{Additional benchmark results.}
Results are computed over $n=5$ random seeds.
Leukemia and TCGA report test accuracy in percent, where higher is better; QM9 reports validation loss, where lower is better. SGD did not converge for QM9.}
\label{tab:additional_results_wrap}
\scriptsize
\setlength{\tabcolsep}{3.5pt}
\renewcommand{\arraystretch}{1.12}
\resizebox{\linewidth}{!}{%
\begin{tabular}{lccc}
\toprule
\textbf{Method}
& \makecell{\textbf{Leukemia}\\\textbf{Acc.} $\uparrow$}
& \makecell{\textbf{TCGA}\\\textbf{Acc.} $\uparrow$}
& \makecell{\textbf{QM9}\\\textbf{Val. Loss} $\downarrow$} \\
\midrule
Adam
& $57.2 \pm 5.9$
& $89.9 \pm 6.4$
& $0.28 \pm 0.15$ \\

AdamW
& $56.9 \pm 6.2$
& $89.8 \pm 6.4$
& $0.35 \pm 0.21$ \\

Muon
& $56.2 \pm 8.6$
& $66.4 \pm 27.3$
& $0.38 \pm 0.08$ \\

SGD
& $57.2 \pm 6.5$
& $89.1 \pm 6.0$
& {--} \\
\midrule
\texttt{GROKtimizer}
& $\mathbf{80.3 \pm 1.4}$
& $\mathbf{97.5 \pm 1.8}$
& $\mathbf{0.16 \pm 0.11}$ \\
\bottomrule
\end{tabular}%
}
\vspace{-10pt}
\end{wraptable}

\textbf{Transfer to real-world high-dimensional datasets.}
We next evaluate whether the same biphasic optimization principle transfers beyond synthetic grokking benchmarks. We consider two high-dimensional biomedical classification tasks, Leukemia and TCGA, using logistic regression models, and one molecular property prediction task, QM9, using a graph neural network. Table~\ref{tab:additional_results_wrap} reports results over $n=5$ random seeds. \texttt{GROKtimizer} achieves the best mean performance on all three datasets.  The strongest gains occur on the biomedical datasets, where the number of features is large relative to the number of samples. This is consistent with our hypothesis that explicit norm minimization is most useful when the training data underdetermines the solution and many interpolating predictors remain compatible with the observed labels. 

\begin{wrapfigure}{r}{0.5\textwidth}
\vspace{-12pt}
    \centering
    \includegraphics[width=0.5\textwidth]{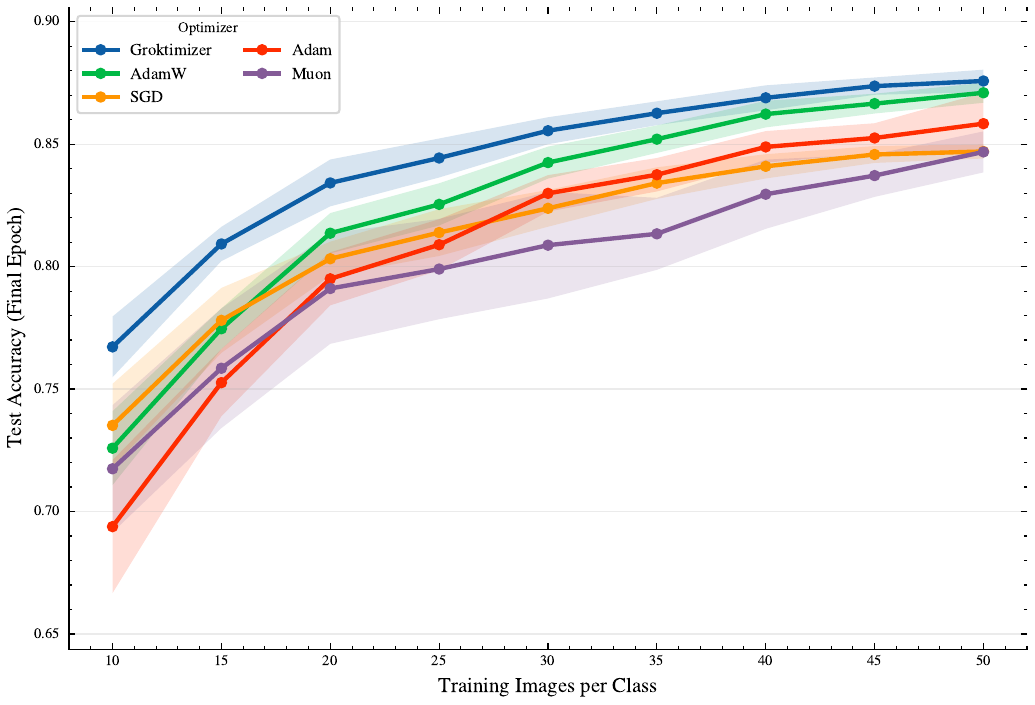}
    \caption{\textbf{MNIST sample efficiency under data scarcity.}
    Test accuracy as the number of training examples per class increases from $10$ to $50$. Results averaged across $n=10$ seeds.}
    \label{fig:scarcity}
\vspace{-10pt}
\end{wrapfigure}
\textbf{Sample efficiency in low-data regimes.}
We next test whether the post-interpolation norm-minimization phase improves sample efficiency. We construct low-data MNIST classification tasks by varying the number of training examples per class from $10$ to $50$, while keeping the architecture and training budget fixed across optimizers. We use an overparameterized MLP with approximately $1.46$ million parameters, creating a regime in which the model can interpolate the training data but the choice of interpolating solution strongly affects generalization. Figure~\ref{fig:scarcity} shows that \texttt{GROKtimizer} provides the largest advantage in the most data-scarce setting. As the number of samples per class increases, the gap between optimizers becomes smaller, consistent with the solution space becoming more constrained by data.

\paragraph{Language-model pretraining.}
Finally, we study whether the biphasic optimization pattern persists when moving from small synthetic and tabular settings to transformer-based language-model pretraining. We train two data-constrained language models: a WikiText2 model with 14M parameters and a BabyLMStrictSmall model with 38M parameters. Figure~\ref{fig:llm_pretraining} shows that the relationship between parameter norm and validation performance is less direct in language modeling than in the preceding benchmarks. In particular, Adam can reach lower parameter norms while still obtaining worse validation loss. This is consistent with the additional complications of transformer pretraining, where scale symmetries, adaptive updates, embedding layers, and noisy next-token objectives weaken a simple norm-generalization interpretation. Nevertheless, we still observe a biphasic optimization pattern: after the model reaches low training loss, switching to the norm-minimization phase changes the norm trajectory and produces a small further reduction in validation loss. The results also suggest that the SGD-based interpolation phase without weight decay is better suited to this setting, whereas applying weight decay too early can mask the data signal and hinder fitting.

\section{Discussion}
\begin{wrapfigure}{r}{0.5\textwidth}
\vspace{-12pt}
    \centering
    \includegraphics[width=\linewidth]{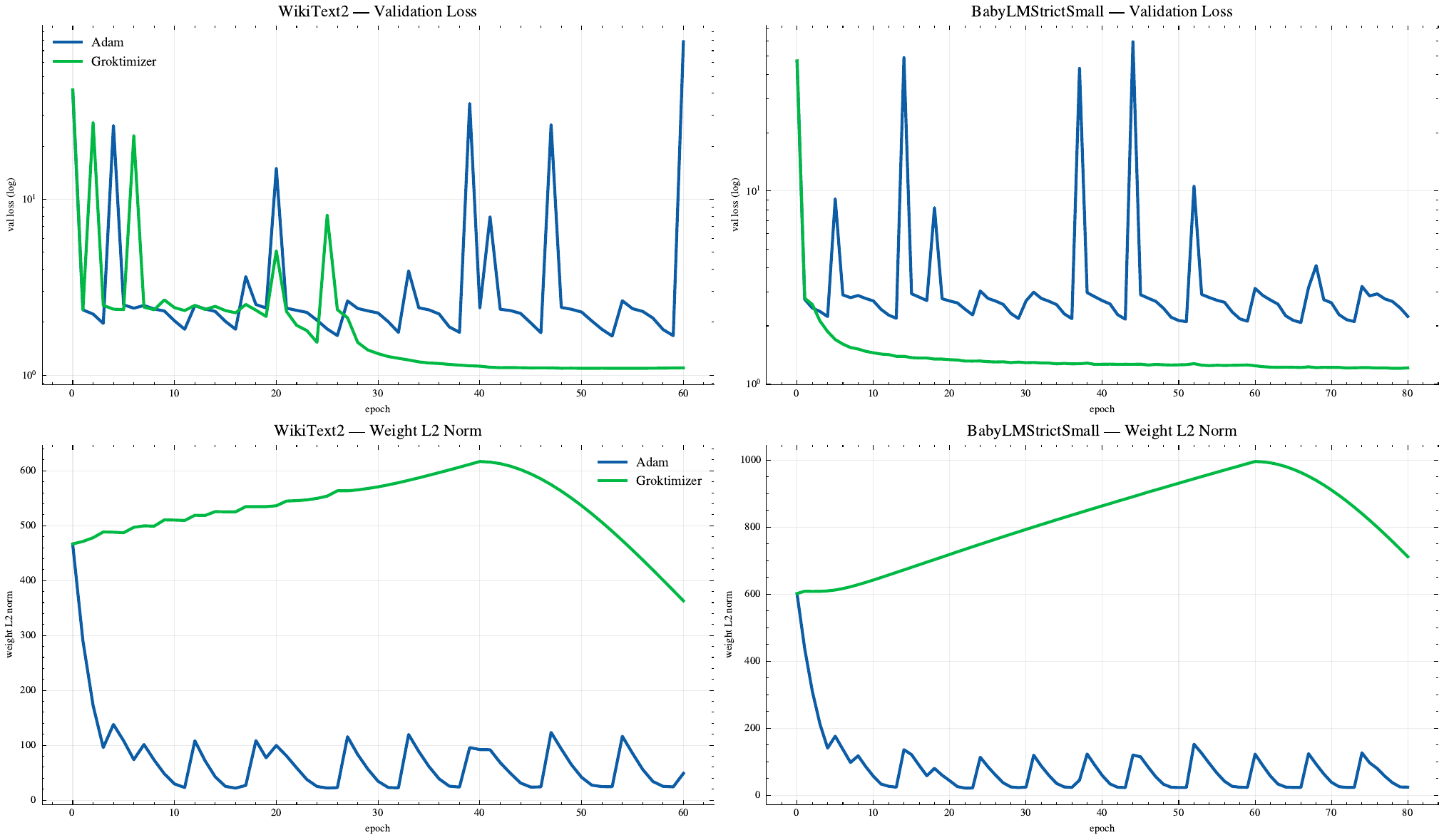}
    \caption{\textbf{Language-model pretraining under data constraints.}
    Validation-loss and parameter norm trajectories for Adam and \texttt{GROKtimizer}.}
    \label{fig:llm_pretraining}
\vspace{-10pt}
\end{wrapfigure}
This work studies model training in high-dimensional, low-sample regimes, where interpolation is often easy but generalization remains difficult. Motivated by the two-phase structure observed in grokking, we proposed \texttt{GROKtimizer}, a biphasic optimization schedule that separates rapid empirical risk minimization from post-interpolation complexity minimization. The central idea is that these two phases impose different optimization requirements: fitting the training data benefits from fast, unregularized movement toward interpolation, whereas generalization after interpolation requires efficient movement within the set of training-compatible solutions toward lower-complexity predictors. Our theoretical analysis formalizes this intuition under a local quadratic approximation of the post-interpolation loss landscape. In this regime, the regularized dynamics along flat Hessian directions reduce to a damped oscillator, and critically damped momentum accelerates convergence toward the low-norm solution relative to gradient descent. 

Empirically, this perspective results in gains observed across our synthetic grokking benchmarks and real-world data-constrained settings. Finally, our language-model pretraining experiments show that the biphasic optimization pattern can still be observed in transformer architectures, but also reveal that the relationship between parameter norm and validation loss becomes less direct.

\paragraph{Limitations.}
The scope of this study is  focused on data-constrained learning. This regime is central in many scientific and biomedical applications, however, it is less representative of large-scale vision or language pretraining, where data availability is often the dominant driver of performance. A second limitation is that our method assumes the model has sufficient capacity to reach interpolation. The proposed schedule is meaningful only once there exists a post-interpolation regime in which the empirical loss no longer strongly distinguishes between candidate solutions. Third, our theoretical analysis relies on a local quadratic approximation of the regularized loss landscape after interpolation --- albeit empirically relaxed to slow evolution of the eigenbasis. Finally, our approach treats parameter norm as an operational proxy for complexity. 

\paragraph{Connection to flat minima hypothesis.}
 The correlation between sharpness of minima and generalization has a long history in Deep Learning \cite{kaddour2022flat, hochreiter1997flat}. However, naive measures of sharpness can be affected by reparameterization and scale symmetries \cite{dinh2017sharp}.
%In our framework, flatness enters through the null space, or near-null space, of the Hessian after interpolation. These directions define a local region in which the empirical loss changes little, allowing the optimizer to move while preserving training performance. Weight decay then selects a lower-norm point within this approximately flat region. Thus, \texttt{GROKtimizer} can be viewed as an explicit procedure for exploiting post-interpolation flat directions: rather than merely observing that flat minima are associated with generalization, it uses the geometry of the flat region to accelerate movement toward a simpler interpolating solution.
In our analysis, we have employed the null space of the Hessian matrix as the domain where the norm (and thus complexity) minimization happens; when the surface of interpolating solutions has higher latent dimension, the optimization domain becomes larger leading to an higher number of low complexity solutions. We thus argue that using the dimensionality of the null space of the Hessian as measure of flatness may be, in future works, a successful strategy. In fact, along with the  norm minimization domain argument, the dimensionality of the kernel of the Hessian is invariant to reparametrization in critical points for most loss functions \cite{lee2006riemannian}. We believe that this is a promising direction, and a possible future investigation.

\printbibliography

@misc{power2022grokking,
  doi = {10.48550/ARXIV.2201.02177},
  url = {https://arxiv.org/abs/2201.02177},
  author = {Power,  Alethea and Burda,  Yuri and Edwards,  Harri and Babuschkin,  Igor and Misra,  Vedant},
  title = {Grokking: Generalization Beyond Overfitting on Small Algorithmic Datasets},
  publisher = {arXiv},
  year = {2022},
  copyright = {arXiv.org perpetual,  non-exclusive license}
}

@article{jacot2018neural,
  title={Neural tangent kernel: Convergence and generalization in neural networks},
  author={Jacot, Arthur and Gabriel, Franck and Hongler, Cl{\'e}ment},
  journal={Advances in neural information processing systems},
  volume={31},
  year={2018}
}

@article{martin2025setol,
  title={SETOL: A Semi-Empirical Theory of (Deep) Learning},
  author={Martin, Charles H and Hinrichs, Christopher},
  journal={arXiv preprint arXiv:2507.17912},
  year={2025}
}

@book{lee2006riemannian,
  title={Riemannian manifolds: an introduction to curvature},
  author={Lee, John M},
  year={2006},
  publisher={Springer Science \& Business Media}
}

@inproceedings{dinh2017sharp,
  title={Sharp minima can generalize for deep nets},
  author={Dinh, Laurent and Pascanu, Razvan and Bengio, Samy and Bengio, Yoshua},
  booktitle={International Conference on Machine Learning},
  pages={1019--1028},
  year={2017},
  organization={PMLR}
}

@article{kaddour2022flat,
  title={When do flat minima optimizers work?},
  author={Kaddour, Jean and Liu, Linqing and Silva, Ricardo and Kusner, Matt J},
  journal={Advances in Neural Information Processing Systems},
  volume={35},
  pages={16577--16595},
  year={2022}
}

@book{nesterov2013introductory,
  title={Introductory lectures on convex optimization: A basic course},
  author={Nesterov, Yurii},
  volume={87},
  year={2013},
  publisher={Springer Science \& Business Media}
}

@article{golub1999molecular,
  title={Molecular classification of cancer: class discovery and class prediction by gene expression monitoring},
  author={Golub, Todd R and Slonim, Donna K and Tamayo, Pablo and Huard, Christine and Gaasenbeek, Michelle and Mesirov, Jill P and Coller, Hilary and Loh, Mignon L and Downing, James R and Caligiuri, Mark A and others},
  journal={science},
  volume={286},
  number={5439},
  pages={531--537},
  year={1999},
  publisher={American Association for the Advancement of Science}
}

@article{prakash2025grokking,
  title={Grokking and Generalization Collapse: Insights from HTSR theory},
  author={Prakash, Hari K and Martin, Charles H},
  journal={arXiv preprint arXiv:2506.04434},
  year={2025}
}

@article{lee2019wide,
  title={Wide neural networks of any depth evolve as linear models under gradient descent},
  author={Lee, Jaehoon and Xiao, Lechao and Schoenholz, Samuel and Bahri, Yasaman and Novak, Roman and Sohl-Dickstein, Jascha and Pennington, Jeffrey},
  journal={Advances in neural information processing systems},
  volume={32},
  year={2019}
}

@inproceedings{allen2019convergence,
  title={A convergence theory for deep learning via over-parameterization},
  author={Allen-Zhu, Zeyuan and Li, Yuanzhi and Song, Zhao},
  booktitle={International conference on machine learning},
  pages={242--252},
  year={2019},
  organization={PMLR}
}

@article{hochreiter1997flat,
  title={Flat minima},
  author={Hochreiter, Sepp and Schmidhuber, J{\"u}rgen},
  journal={Neural computation},
  volume={9},
  number={1},
  pages={1--42},
  year={1997},
  publisher={MIT Press One Rogers Street, Cambridge, MA 02142-1209, USA journals-info~…}
}

@article{bartlett2020benign,
  title={Benign overfitting in linear regression},
  author={Bartlett, Peter L and Long, Philip M and Lugosi, G{\'a}bor and Tsigler, Alexander},
  journal={Proceedings of the National Academy of Sciences},
  volume={117},
  number={48},
  pages={30063--30070},
  year={2020},
  publisher={National Academy of Sciences}
}

@article{jeffares2025not,
  title={Not All Explanations for Deep Learning Phenomena Are Equally Valuable},
  author={Jeffares, Alan and van der Schaar, Mihaela},
  journal={arXiv preprint arXiv:2506.23286},
  year={2025}
}

@article{liu2022towards,
  title={Towards understanding grokking: An effective theory of representation learning},
  author={Liu, Ziming and Kitouni, Ouail and Nolte, Niklas S and Michaud, Eric and Tegmark, Max and Williams, Mike},
  journal={Advances in Neural Information Processing Systems},
  volume={35},
  pages={34651--34663},
  year={2022}
}

@article{balasubramanian1997statistical,
  title={Statistical inference, Occam's razor, and statistical mechanics on the space of probability distributions},
  author={Balasubramanian, Vijay},
  journal={Neural computation},
  volume={9},
  number={2},
  pages={349--368},
  year={1997},
  publisher={MIT Press One Rogers Street, Cambridge, MA 02142-1209, USA journals-info~…}
}

@inproceedings{
kumar2024grokking,
title={Grokking as the transition from lazy to rich training dynamics},
author={Tanishq Kumar and Blake Bordelon and Samuel J. Gershman and Cengiz Pehlevan},
booktitle={The Twelfth International Conference on Learning Representations},
year={2024},
url={https://openreview.net/forum?id=vt5mnLVIVo}
}

@inproceedings{
boursier2025theoretical,
title={A Theoretical Framework for Grokking: Interpolation followed by Riemannian Norm Minimisation},
author={Etienne Boursier and Scott Pesme and Radu-Alexandru Dragomir},
booktitle={The Thirty-ninth Annual Conference on Neural Information Processing Systems},
year={2026},
url={https://openreview.net/forum?id=iSvAAHGFSw}
}

@inproceedings{draxler2018essentially,
  title={Essentially no barriers in neural network energy landscape},
  author={Draxler, Felix and Veschgini, Kambis and Salmhofer, Manfred and Hamprecht, Fred},
  booktitle={International conference on machine learning},
  pages={1309--1318},
  year={2018},
  organization={PMLR}
}

@article{hastie2022surprises,
  title={Surprises in high-dimensional ridgeless least squares interpolation},
  author={Hastie, Trevor and Montanari, Andrea and Rosset, Saharon and Tibshirani, Ryan J},
  journal={Annals of statistics},
  volume={50},
  number={2},
  pages={949},
  year={2022}
}

@inproceedings{benton2021loss,
  title={Loss surface simplexes for mode connecting volumes and fast ensembling},
  author={Benton, Gregory and Maddox, Wesley and Lotfi, Sanae and Wilson, Andrew Gordon Gordon},
  booktitle={International Conference on Machine Learning},
  pages={769--779},
  year={2021},
  organization={PMLR}
}

@article{garipov2018loss,
  title={Loss surfaces, mode connectivity, and fast ensembling of dnns},
  author={Garipov, Timur and Izmailov, Pavel and Podoprikhin, Dmitrii and Vetrov, Dmitry P and Wilson, Andrew G},
  journal={Advances in neural information processing systems},
  volume={31},
  year={2018}
}

@misc{thilak2022slingshot,
  doi = {10.48550/ARXIV.2206.04817},
  url = {https://arxiv.org/abs/2206.04817},
  author = {Thilak,  Vimal and Littwin,  Etai and Zhai,  Shuangfei and Saremi,  Omid and Paiss,  Roni and Susskind,  Joshua},
  title = {The Slingshot Mechanism: An Empirical Study of Adaptive Optimizers and the Grokking Phenomenon},
  publisher = {arXiv},
  year = {2022},
  copyright = {arXiv.org perpetual,  non-exclusive license}
}

@inproceedings{
nanda2023fourier,
title={Progress measures for grokking via mechanistic interpretability},
author={Neel Nanda and Lawrence Chan and Tom Lieberum and Jess Smith and Jacob Steinhardt},
booktitle={The Eleventh International Conference on Learning Representations },
year={2023},
url={https://openreview.net/forum?id=9XFSbDPmdW}
}

@article{kumar2023grokking,
  title={Mechanisms of Grokking: Loss Landscape Separation and Negative Learning Momentum},
  author={Kumar, Tanishq and others},
  journal={arXiv preprint},
  year={2023}
}

@misc{tveit2025muon,
  doi = {10.48550/ARXIV.2504.16041},
  url = {https://arxiv.org/abs/2504.16041},
  author = {Tveit,  Amund and Remseth,  Bjørn and Skogvold,  Arve},
  title = {Muon Optimizer Accelerates Grokking},
  publisher = {arXiv},
  year = {2025},
  copyright = {Creative Commons Attribution 4.0 International}
}

@misc{lee2024grokfast,
  doi = {10.48550/ARXIV.2405.20233},
  url = {https://arxiv.org/abs/2405.20233},
  author = {Lee,  Jaerin and Kang,  Bong Gyun and Kim,  Kihoon and Lee,  Kyoung Mu},
  title = {Grokfast: Accelerated Grokking by Amplifying Slow Gradients},
  publisher = {arXiv},
  year = {2024},
  copyright = {arXiv.org perpetual,  non-exclusive license}
}

@misc{zhou2025neuralgrok,
  doi = {10.48550/ARXIV.2504.17243},
  url = {https://arxiv.org/abs/2504.17243},
  author = {Zhou,  Xinyu and Fan,  Simin and Jaggi,  Martin and Fu,  Jie},
  title = {NeuralGrok: Accelerate Grokking by Neural Gradient Transformation},
  publisher = {arXiv},
  year = {2025},
  copyright = {Creative Commons Attribution Share Alike 4.0 International}
}

@misc{musat2025geometry,
  doi = {10.48550/ARXIV.2511.01938},
  url = {https://arxiv.org/abs/2511.01938},
  author = {Musat,  Tiberiu},
  title = {The Geometry of Grokking: Norm Minimization on the Zero-Loss Manifold},
  publisher = {arXiv},
  year = {2025},
  copyright = {Creative Commons Attribution 4.0 International}
}

@article{park2025acceleration,
  title = {Acceleration of grokking in learning arithmetic operations via Kolmogorov–Arnold representation},
  volume = {640},
  ISSN = {0925-2312},
  url = {http://dx.doi.org/10.1016/j.neucom.2025.130347},
  DOI = {10.1016/j.neucom.2025.130347},
  journal = {Neurocomputing},
  publisher = {Elsevier BV},
  author = {Park,  Yeachan and Kim,  Minseok and Kim,  Yeoneung},
  year = {2025},
  month = Aug,
  pages = {130347}
}

@inproceedings{
abramov2025grokking,
title={Grokking in the Wild: Data Augmentation for Real-World Multi-Hop Reasoning with Transformers},
author={Roman Abramov and Felix Steinbauer and Gjergji Kasneci},
booktitle={Forty-second International Conference on Machine Learning},
year={2025},
url={https://openreview.net/forum?id=lyUJH51URt}
}

@misc{varma2023explaining,
  doi = {10.48550/ARXIV.2309.02390},
  url = {https://arxiv.org/abs/2309.02390},
  author = {Varma,  Vikrant and Shah,  Rohin and Kenton,  Zachary and Kramár,  János and Kumar,  Ramana},
  title = {Explaining grokking through circuit efficiency},
  publisher = {arXiv},
  year = {2023},
  copyright = {arXiv.org perpetual,  non-exclusive license}
}

@inproceedings{
soudry2018the,
title={The Implicit Bias of Gradient Descent on Separable Data},
author={Daniel Soudry and Elad Hoffer and Nathan Srebro},
booktitle={International Conference on Learning Representations},
year={2018},
url={https://openreview.net/forum?id=r1q7n9gAb},
}

@inproceedings{gunasekar2018implicit,
 author = {Gunasekar, Suriya and Lee, Jason D and Soudry, Daniel and Srebro, Nati},
 booktitle = {Advances in Neural Information Processing Systems},
 editor = {S. Bengio and H. Wallach and H. Larochelle and K. Grauman and N. Cesa-Bianchi and R. Garnett},
 pages = {},
 publisher = {Curran Associates, Inc.},
 title = {Implicit Bias of Gradient Descent on Linear Convolutional Networks},
 url = {https://proceedings.neurips.cc/paper\_files/paper/2018/file/0e98aeeb54acf612b9eb4e48a269814c-Paper.pdf},
 volume = {31},
 year = {2018}
}

@misc{neyshabur2017implicit,
  doi = {10.48550/ARXIV.1709.01953},
  url = {https://arxiv.org/abs/1709.01953},
  author = {Neyshabur,  Behnam},
  title = {Implicit Regularization in Deep Learning},
  publisher = {arXiv},
  year = {2017},
  copyright = {arXiv.org perpetual,  non-exclusive license}
}

@article{Xu2023,
  title = {Small data machine learning in materials science},
  volume = {9},
  ISSN = {2057-3960},
  url = {http://dx.doi.org/10.1038/s41524-023-01000-z},
  DOI = {10.1038/s41524-023-01000-z},
  number = {1},
  journal = {npj Computational Materials},
  publisher = {Springer Science and Business Media LLC},
  author = {Xu,  Pengcheng and Ji,  Xiaobo and Li,  Minjie and Lu,  Wencong},
  year = {2023},
  month = Mar 
}

@article{Dou2023,
  title = {Machine Learning Methods for Small Data Challenges in Molecular Science},
  volume = {123},
  ISSN = {1520-6890},
  url = {http://dx.doi.org/10.1021/acs.chemrev.3c00189},
  DOI = {10.1021/acs.chemrev.3c00189},
  number = {13},
  journal = {Chemical Reviews},
  publisher = {American Chemical Society (ACS)},
  author = {Dou,  Bozheng and Zhu,  Zailiang and Merkurjev,  Ekaterina and Ke,  Lu and Chen,  Long and Jiang,  Jian and Zhu,  Yueying and Liu,  Jie and Zhang,  Bengong and Wei,  Guo-Wei},
  year = {2023},
  month = June,
  pages = {8736–8780}
}

@article{FeldnerBusztin2023,
  title = {Dealing with dimensionality: the application of machine learning to multi-omics data},
  volume = {39},
  ISSN = {1367-4811},
  url = {http://dx.doi.org/10.1093/bioinformatics/btad021},
  DOI = {10.1093/bioinformatics/btad021},
  number = {2},
  journal = {Bioinformatics},
  publisher = {Oxford University Press (OUP)},
  author = {Feldner-Busztin,  Dylan and Firbas Nisantzis,  Panos and Edmunds,  Shelley Jane and Boza,  Gergely and Racimo,  Fernando and Gopalakrishnan,  Shyam and Limborg,  Morten Tønsberg and Lahti,  Leo and de Polavieja,  Gonzalo G},
  editor = {Wren,  Jonathan},
  year = {2023},
  month = Jan 
}

@article{Rajkomar2018,
  title = {Scalable and accurate deep learning with electronic health records},
  volume = {1},
  ISSN = {2398-6352},
  url = {http://dx.doi.org/10.1038/s41746-018-0029-1},
  DOI = {10.1038/s41746-018-0029-1},
  number = {1},
  journal = {npj Digital Medicine},
  publisher = {Springer Science and Business Media LLC},
  author = {Rajkomar,  Alvin and Oren,  Eyal and Chen,  Kai and Dai,  Andrew M. and Hajaj,  Nissan and Hardt,  Michaela and Liu,  Peter J. and Liu,  Xiaobing and Marcus,  Jake and Sun,  Mimi and Sundberg,  Patrik and Yee,  Hector and Zhang,  Kun and Zhang,  Yi and Flores,  Gerardo and Duggan,  Gavin E. and Irvine,  Jamie and Le,  Quoc and Litsch,  Kurt and Mossin,  Alexander and Tansuwan,  Justin and Wang,  De and Wexler,  James and Wilson,  Jimbo and Ludwig,  Dana and Volchenboum,  Samuel L. and Chou,  Katherine and Pearson,  Michael and Madabushi,  Srinivasan and Shah,  Nigam H. and Butte,  Atul J. and Howell,  Michael D. and Cui,  Claire and Corrado,  Greg S. and Dean,  Jeffrey},
  year = {2018},
  month = May 
}

@article{CHINCO2018,
  title = {Sparse Signals in the Cross‐Section of Returns},
  volume = {74},
  ISSN = {1540-6261},
  url = {http://dx.doi.org/10.1111/jofi.12733},
  DOI = {10.1111/jofi.12733},
  number = {1},
  journal = {The Journal of Finance},
  publisher = {Wiley},
  author = {Chinco,  Alex and Clark‐Joseph,  Adam D. and Ye,  Mao},
  year = {2018},
  month = Nov,
  pages = {449–492}
}

@inproceedings{
zhang2017understanding,
title={Understanding deep learning requires rethinking generalization},
author={Chiyuan Zhang and Samy Bengio and Moritz Hardt and Benjamin Recht and Oriol Vinyals},
booktitle={International Conference on Learning Representations},
year={2017},
url={https://openreview.net/forum?id=Sy8gdB9xx}
}

@inproceedings{Garipov2018,
 author = {Garipov, Timur and Izmailov, Pavel and Podoprikhin, Dmitrii and Vetrov, Dmitry and Wilson, Andrew G},
 booktitle = {Advances in Neural Information Processing Systems},
 editor = {S. Bengio and H. Wallach and H. Larochelle and K. Grauman and N. Cesa-Bianchi and R. Garnett},
 pages = {},
 publisher = {Curran Associates, Inc.},
 title = {Loss Surfaces, Mode Connectivity, and Fast Ensembling of DNNs},
 url = {https://proceedings.neurips.cc/paper\_files/paper/2018/file/be3087e74e9100d4bc4c6268cdbe8456-Paper.pdf},
 volume = {31},
 year = {2018}
}

@article{Soudry2018,
  author  = {Daniel Soudry and Elad Hoffer and Mor Shpigel Nacson and Suriya Gunasekar and Nathan Srebro},
  title   = {The Implicit Bias of Gradient Descent on Separable Data},
  journal = {Journal of Machine Learning Research},
  year    = {2018},
  volume  = {19},
  number  = {70},
  pages   = {1--57},
  url     = {http://jmlr.org/papers/v19/18-188.html}
}

@book{Rasmussen2005,
  title = {Gaussian Processes for Machine Learning},
  ISBN = {9780262256834},
  url = {http://dx.doi.org/10.7551/mitpress/3206.001.0001},
  DOI = {10.7551/mitpress/3206.001.0001},
  publisher = {The MIT Press},
  author = {Rasmussen,  Carl Edward and Williams,  Christopher K. I.},
  year = {2005},
  month = Nov 
}

@article{Belkin2019,
  title = {Reconciling modern machine-learning practice and the classical bias–variance trade-off},
  volume = {116},
  ISSN = {1091-6490},
  url = {http://dx.doi.org/10.1073/pnas.1903070116},
  DOI = {10.1073/pnas.1903070116},
  number = {32},
  journal = {Proceedings of the National Academy of Sciences},
  publisher = {Proceedings of the National Academy of Sciences},
  author = {Belkin,  Mikhail and Hsu,  Daniel and Ma,  Siyuan and Mandal,  Soumik},
  year = {2019},
  month = July,
  pages = {15849–15854}
}

@article{Hochreiter1997,
  title = {Flat Minima},
  volume = {9},
  ISSN = {1530-888X},
  url = {http://dx.doi.org/10.1162/neco.1997.9.1.1},
  DOI = {10.1162/neco.1997.9.1.1},
  number = {1},
  journal = {Neural Computation},
  publisher = {MIT Press},
  author = {Hochreiter,  Sepp and Schmidhuber,  J\"{u}rgen},
  year = {1997},
  month = Jan,
  pages = {1–42}
}

@inproceedings{Krogh1991,
 author = {Krogh, Anders and Hertz, John},
 booktitle = {Advances in Neural Information Processing Systems},
 editor = {J. Moody and S. Hanson and R.P. Lippmann},
 pages = {},
 publisher = {Morgan-Kaufmann},
 title = {A Simple Weight Decay Can Improve Generalization},
 url = {https://proceedings.neurips.cc/paper\_files/paper/1991/file/8eefcfdf5990e441f0fb6f3fad709e21-Paper.pdf},
 volume = {4},
 year = {1991}
}

@article{Polyak1964,
  title = {Some methods of speeding up the convergence of iteration methods},
  volume = {4},
  ISSN = {0041-5553},
  url = {http://dx.doi.org/10.1016/0041-5553(64)90137-5},
  DOI = {10.1016/0041-5553(64)90137-5},
  number = {5},
  journal = {USSR Computational Mathematics and Mathematical Physics},
  publisher = {Elsevier BV},
  author = {Polyak,  B.T.},
  year = {1964},
  month = Jan,
  pages = {1–17}
}

@article{Nesterov1983AMF,
  title={A method for solving the convex programming problem with convergence rate $O(1/k^2)$},
  author={Yurii Nesterov},
  journal={Proceedings of the USSR Academy of Sciences},
  year={1983},
  volume={269},
  pages={543-547},
  url={https://api.semanticscholar.org/CorpusID:145918791}
}

@InProceedings{pmlr-v28-sutskever13,
  title = 	 {On the importance of initialization and momentum in deep learning},
  author = 	 {Sutskever, Ilya and Martens, James and Dahl, George and Hinton, Geoffrey},
  booktitle = 	 {Proceedings of the 30th International Conference on Machine Learning},
  pages = 	 {1139--1147},
  year = 	 {2013},
  editor = 	 {Dasgupta, Sanjoy and McAllester, David},
  volume = 	 {28},
  number =       {3},
  series = 	 {Proceedings of Machine Learning Research},
  address = 	 {Atlanta, Georgia, USA},
  month = 	 {17--19 Jun},
  publisher =    {PMLR},
  url = 	 {https://proceedings.mlr.press/v28/sutskever13.html}
}

@inproceedings{kingma2015adam,
  author    = {Kingma, Diederik P. and Ba, Jimmy},
  title     = {Adam: A Method for Stochastic Optimization},
  booktitle = {International Conference on Learning Representations (ICLR)},
  year      = {2015},
  url       = {https://arxiv.org/abs/1412.6980}
}

@inproceedings{Loshchilov2017DecoupledWD,
  title={Decoupled Weight Decay Regularization},
  author={Ilya Loshchilov and Frank Hutter},
  booktitle={International Conference on Learning Representations},
  year={2017},
  url={https://api.semanticscholar.org/CorpusID:53592270}
}

@book{Ogata2001,
author = {Ogata, Katsuhiko},
title = {Modern Control Engineering},
year = {2001},
isbn = {0130609072},
publisher = {Prentice Hall PTR},
address = {USA},
edition = {4th}
}

@article{JMLR:v12:duchi11a,
  author  = {John Duchi and Elad Hazan and Yoram Singer},
  title   = {Adaptive Subgradient Methods for Online Learning and Stochastic Optimization},
  journal = {Journal of Machine Learning Research},
  year    = {2011},
  volume  = {12},
  number  = {61},
  pages   = {2121--2159},
  url     = {http://jmlr.org/papers/v12/duchi11a.html}
}

@article{JMLR:v15:srivastava14a,
  author  = {Nitish Srivastava and Geoffrey Hinton and Alex Krizhevsky and Ilya Sutskever and Ruslan Salakhutdinov},
  title   = {Dropout: A Simple Way to Prevent Neural Networks from Overfitting},
  journal = {Journal of Machine Learning Research},
  year    = {2014},
  volume  = {15},
  number  = {56},
  pages   = {1929--1958},
  url     = {http://jmlr.org/papers/v15/srivastava14a.html}
}

@inproceedings{Caruana2000,
 author = {Caruana, Rich and Lawrence, Steve and Giles, C.},
 booktitle = {Advances in Neural Information Processing Systems},
 editor = {T. Leen and T. Dietterich and V. Tresp},
 pages = {},
 publisher = {MIT Press},
 title = {Overfitting in Neural Nets: Backpropagation, Conjugate Gradient, and Early Stopping},
 url = {https://proceedings.neurips.cc/paper\_files/paper/2000/file/059fdcd96baeb75112f09fa1dcc740cc-Paper.pdf},
 volume = {13},
 year = {2000}
}

@article{Yao2007,
  title = {On Early Stopping in Gradient Descent Learning},
  volume = {26},
  ISSN = {1432-0940},
  url = {http://dx.doi.org/10.1007/s00365-006-0663-2},
  DOI = {10.1007/s00365-006-0663-2},
  number = {2},
  journal = {Constructive Approximation},
  publisher = {Springer Science and Business Media LLC},
  author = {Yao,  Yuan and Rosasco,  Lorenzo and Caponnetto,  Andrea},
  year = {2007},
  month = Apr,
  pages = {289–315}
}

@inproceedings{Dinh2017,
author = {Dinh, Laurent and Pascanu, Razvan and Bengio, Samy and Bengio, Yoshua},
title = {Sharp minima can generalize for deep nets},
year = {2017},
publisher = {JMLR.org},
booktitle = {Proceedings of the 34th International Conference on Machine Learning - Volume 70},
pages = {1019–1028},
numpages = {10},
location = {Sydney, NSW, Australia},
series = {ICML'17}
}

@article{cancer2013comprehensive,
  title={Comprehensive molecular characterization of clear cell renal cell carcinoma},
  author={TCGA Research Network},
  journal={Nature},
  volume={499},
  number={7456},
  pages={43--49},
  year={2013},
  publisher={Nature Publishing Group UK London}
}
\newpage
\appendix

\AppendixBanner
\section{Extended Related Work}

\paragraph{Optimizer families in deep learning.}
Deep-learning optimization is dominated by several families of first-order methods. Stochastic gradient descent (SGD) and minibatch SGD remain fundamental baselines because of their simplicity, scalability, and often strong generalization properties. Momentum variants, including heavy-ball momentum and Nesterov acceleration, add velocity terms that smooth stochastic gradients and accelerate convergence in ill-conditioned landscapes \cite{Polyak1964,Nesterov1983AMF}. Adaptive methods such as AdaGrad, RMSProp, Adam, and AdamW rescale updates using coordinate-wise gradient statistics, improving robustness across heterogeneous parameter scales \cite{JMLR:v12:duchi11a,kingma2015adam,Loshchilov2017DecoupledWD}. More recent optimizers such as Muon introduce approximate matrix- or spectral-normalized updates to improve conditioning in neural-network training \cite{tveit2025muon}. These optimizers differ in how they precondition, smooth, or normalize gradients, but they are generally designed as full-training optimizers rather than phase-specific methods specialized for the post-interpolation regime.

\paragraph{Weight decay and norm-based regularization.}
Weight decay is one of the oldest and most widely used regularization mechanisms in neural-network training. It biases optimization toward smaller parameter norms and can improve generalization by suppressing unnecessary degrees of freedom \cite{Krogh1991}. In modern deep learning, the distinction between coupled $L^2$ penalties and decoupled weight decay is important: in adaptive optimizers, adding an $L^2$ penalty to the loss is not equivalent to directly decaying the parameters, motivating methods such as AdamW \cite{Loshchilov2017DecoupledWD}. Norm-based regularization can be understood as imposing a preference over the set of functions or parameterizations compatible with the training data. In the present work, this role becomes especially explicit after interpolation, where the empirical loss no longer distinguishes between solutions and weight decay becomes the main driver of movement toward lower-norm parameters.

\paragraph{Non-weight-decay regularization strategies.}
Many regularization strategies do not directly penalize parameter norm. Dropout randomly removes units during training, discouraging co-adaptation and approximating an implicit ensemble of subnetworks \cite{JMLR:v15:srivastava14a}. Data augmentation expands the effective training distribution by applying label-preserving transformations, thereby constraining the learned function rather than the parameter vector. Early stopping regularizes by terminating training before the optimizer fully exploits high-variance directions \cite{Caruana2000,Yao2007}. Other methods, including label smoothing, mixup, stochastic depth, and noise injection, regularize through the targets, inputs, architecture, or training dynamics. These approaches can improve generalization without explicitly selecting a minimum-norm interpolator. They are therefore complementary to our focus on the post-interpolation behavior induced by weight decay.

\paragraph{Norm-based solution selection.}
Norm-based solution selection is analytically clearest in overparameterized linear regression, where all interpolating solutions can be written as a minimum-norm interpolator plus an arbitrary null-space component. In this setting, moving away from the minimum-norm solution can increase test error even though training error remains zero \cite{hastie2022surprises}. In neural networks, the same idea is more subtle because parameter norms are affected by architecture, scale symmetries, normalization layers, and reparameterizations. Nevertheless, norm remains an operationally meaningful quantity when the optimizer explicitly acts on it through weight decay. Our analysis therefore treats norm not as a universal invariant measure of function complexity, but as the concrete complexity proxy optimized during the second phase of training.

\paragraph{Flat minima and sharpness.}
The flat-minima hypothesis proposes that broad basins of low loss generalize better because their predictions are less sensitive to small parameter perturbations \cite{Hochreiter1997}. This idea has influenced work on sharpness, batch size, and generalization, but it also has important limitations: parameter-space sharpness can change under reparameterizations without changing the represented function \cite{Dinh2017}. Later approaches therefore distinguish between naive sharpness measures and more invariant or normalized notions of flatness. In our setting, flatness is used in a restricted local sense. Flat directions of the Hessian identify where training loss changes weakly after interpolation; weight decay then determines which direction the optimizer moves within this locally flat region.

\paragraph{Mode connectivity and interpolation manifolds.}
Empirical studies of neural-network loss landscapes show that independently trained solutions can often be connected by low-loss paths or higher-dimensional regions of low loss \cite{Garipov2018,draxler2018essentially,benton2021loss}. These results suggest that overparameterized networks do not converge to isolated minima, but to extended sets of functionally similar solutions. This supports a manifold view of interpolation: after reaching zero or near-zero training loss, the model may still move substantially in parameter space while preserving training performance. Our local Hessian-nullspace approximation can be interpreted as a tangent-space approximation to such an interpolation manifold.

\section{Extended experimental results}

\subsection{Training run example}

We provide in Figure~\ref{fig:qm9_training_example} an example training trajectory on the QM9 dataset. \texttt{GROKtimizer} reaches a substantially lower test loss than the baseline optimizers, with the largest improvement arising during the second phase of the schedule, when the optimizer explicitly enters the norm-minimization regime. This behavior is consistent with our central hypothesis: once interpolation has been achieved, continued optimization can still improve the learned solution by moving toward a lower-complexity interpolating predictor. Notably, even when trained for substantially more epochs, the baseline optimizers do not appear to enter the same post-interpolation regime and fail to match the final performance of \texttt{GROKtimizer}.

\begin{figure}[h]
    \centering
    \includegraphics[width=1.0\linewidth]{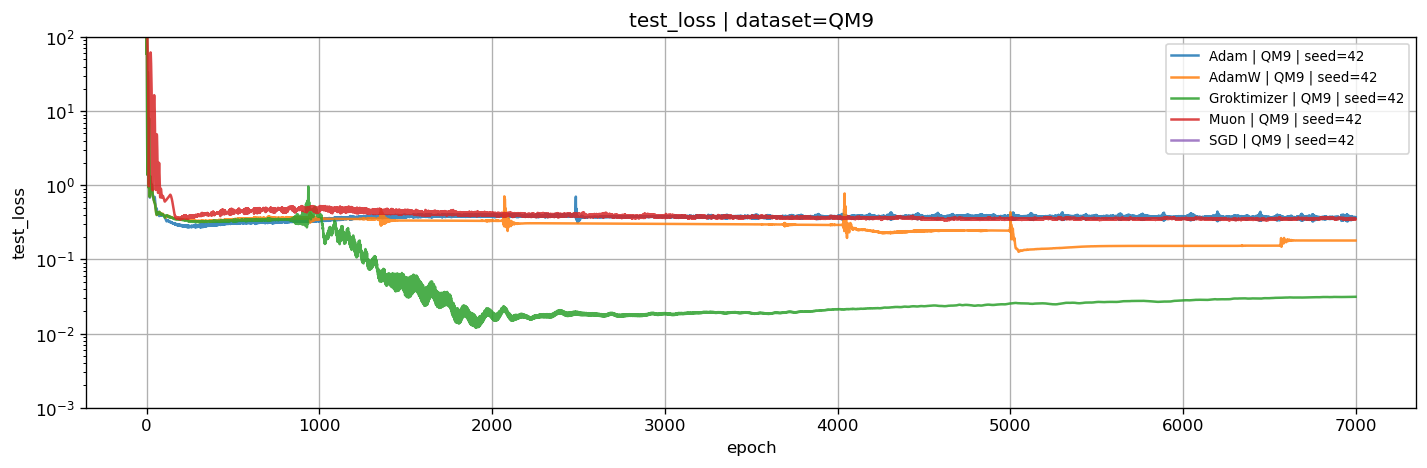}
    \caption{\textbf{Example training run on QM9.} Training dynamics of \texttt{GROKtimizer} and baseline optimizers on the QM9 dataset.}
    \label{fig:qm9_training_example}
\end{figure}

\subsection{Detailed dynamics on the synthetic Gaussian benchmark}

We further analyze the \texttt{GROKtimizer} trajectory on the synthetic Gaussian benchmark, where the relationship between interpolation, parameter norm, and generalization can be inspected directly. Figure~\ref{fig:gaussian_dynamics} shows that \texttt{GROKtimizer} follows a similar initial training-loss decay to the baseline optimizers, indicating that the first phase preserves efficient fitting dynamics. The difference emerges after interpolation: \texttt{GROKtimizer} enters a rapid norm-minimization regime, reducing the parameter norm substantially faster than the baselines. This reduction is accompanied by improved test loss, supporting the interpretation that the second phase guides the model toward a lower-complexity interpolating solution with better generalization.

\begin{figure}[htbp]
    \centering

    \begin{subfigure}{0.32\textwidth}
        \centering
        \includegraphics[width=\linewidth]{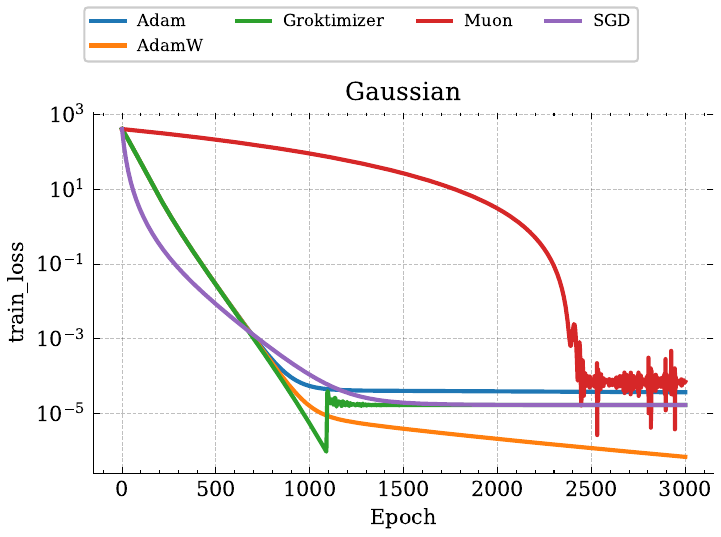}
        \caption{Train loss.}
        \label{fig:gaussian_train_loss}
    \end{subfigure}
    \hfill
    \begin{subfigure}{0.32\textwidth}
        \centering
        \includegraphics[width=\linewidth]{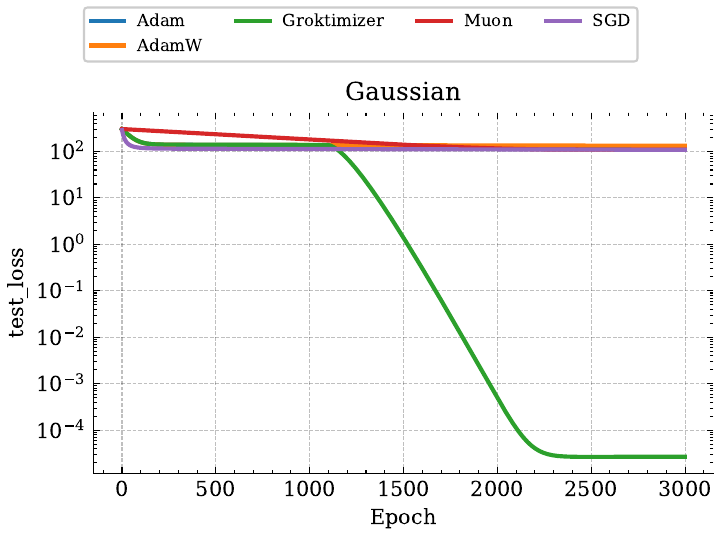}
        \caption{Test loss.}
        \label{fig:gaussian_test_loss}
    \end{subfigure}
    \hfill
    \begin{subfigure}{0.32\textwidth}
        \centering
        \includegraphics[width=\linewidth]{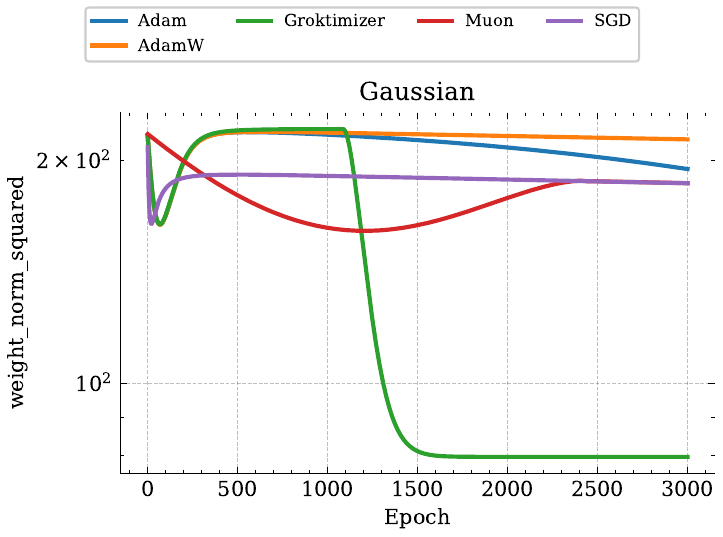}
        \caption{Weight norm.}
        \label{fig:gaussian_weight_norm}
    \end{subfigure}

    \caption{\textbf{Detailed dynamics on the synthetic Gaussian benchmark.}
    \texttt{GROKtimizer} initially follows a training-loss trajectory comparable to the baseline optimizers, but after interpolation it rapidly reduces parameter norm. This post-interpolation norm reduction is accompanied by lower test loss, consistent with the proposed complexity-minimization mechanism.}
    \label{fig:gaussian_dynamics}
\end{figure}

\subsection{Learning rate comparison}

A possible concern is that the acceleration achieved by \texttt{GROKtimizer} is simply a consequence of being able to use a larger learning rate in the second phase. To test this, we compare \texttt{GROKtimizer} against Adam using the same elevated post-interpolation learning rate. As shown in Figure~\ref{fig:phase2lr_comparison}, Adam becomes highly unstable under this setting and fails to achieve comparable performance. This indicates that the gains of \texttt{GROKtimizer} do not arise from a larger learning rate alone, but from the combination of entering the post-interpolation regime, exploiting the flatter loss landscape, and using an appropriate momentum term during the norm-minimization phase. We show this effect on both Binary Addition and Modular Addition.

\begin{figure}[!htb]
    \centering
    \begin{subfigure}{0.49\linewidth}
        \centering
        \includegraphics[width=\linewidth]{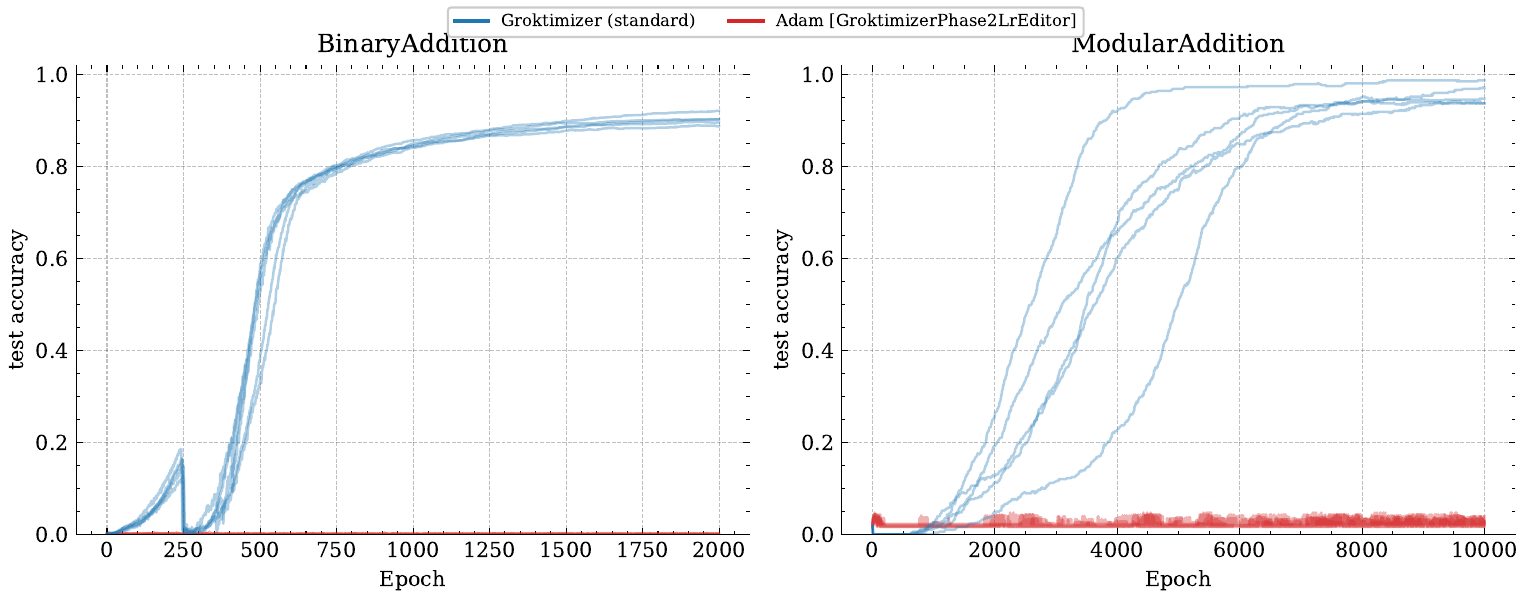}
        \caption{Test accuracy.}
        \label{fig:phase2lr_acc}
    \end{subfigure}
    \hfill
    \begin{subfigure}{0.49\linewidth}
        \centering
        \includegraphics[width=\linewidth]{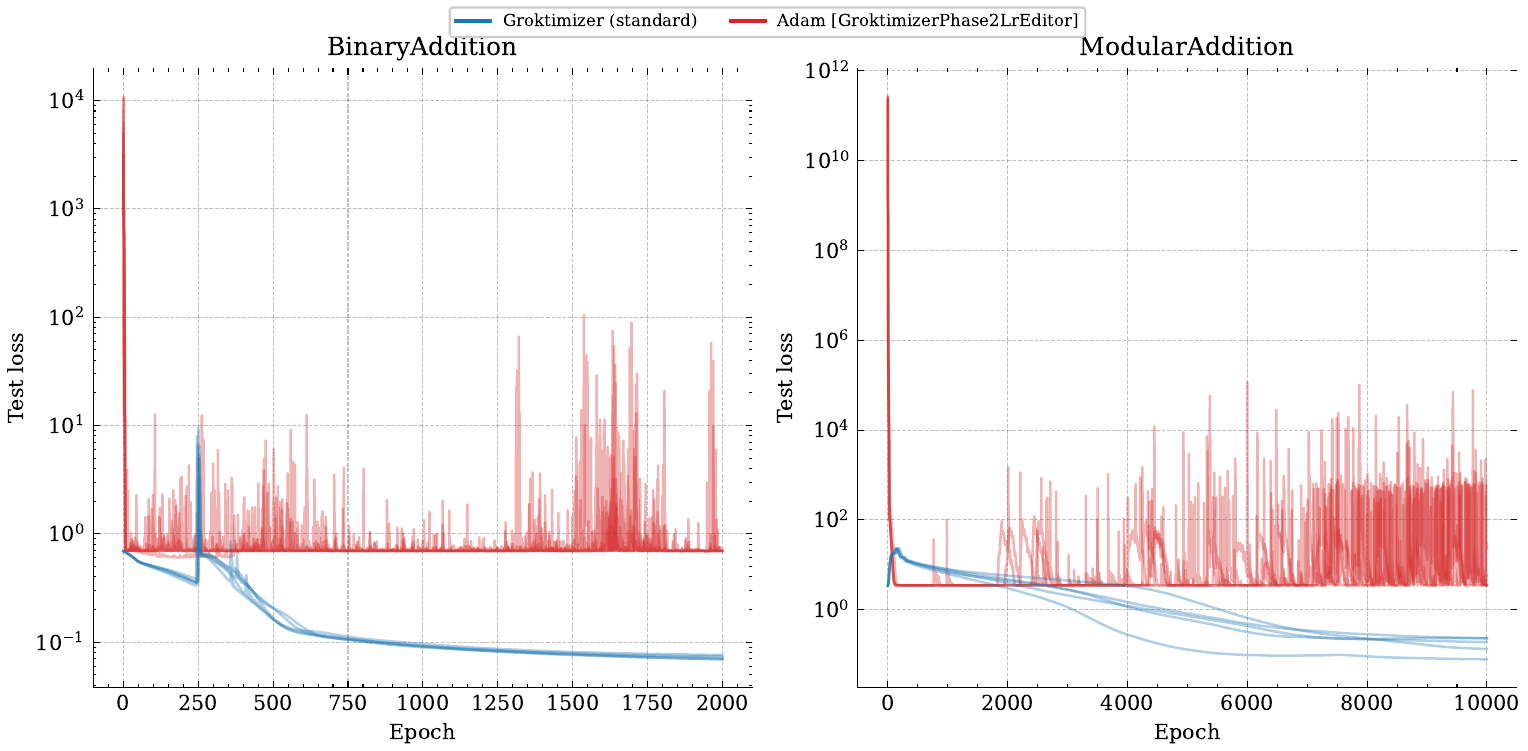}
        \caption{Test loss.}
        \label{fig:phase2lr_loss}
    \end{subfigure}
    \caption{\textbf{Learning-rate comparison after interpolation.} Performance on Binary Addition and Modular Addition when using an increased post-interpolation learning rate. While \texttt{GROKtimizer} remains stable and benefits from the second-phase schedule, Adam becomes unstable when trained with the same elevated learning rate. This shows that the observed speedup is not explained by learning rate alone, but by the phase-specific optimization dynamics of \texttt{GROKtimizer}.}
    \label{fig:phase2lr_comparison}
\end{figure}

\subsection{Examining model behavior during training}

To better understand the dynamics induced by the second phase of \texttt{GROKtimizer}, we examine how the model representation evolves during post-interpolation optimization. We use the neural tangent kernel (NTK) as a representation proxy: for each checkpoint, we compute the NTK features associated with the first output logit on both train and test samples, and visualize the resulting representations after separating samples by class.

\begin{figure}[h]
    \centering
    \includegraphics[width=1.0\linewidth, trim={0cm 0cm 0cm 0cm}, clip]{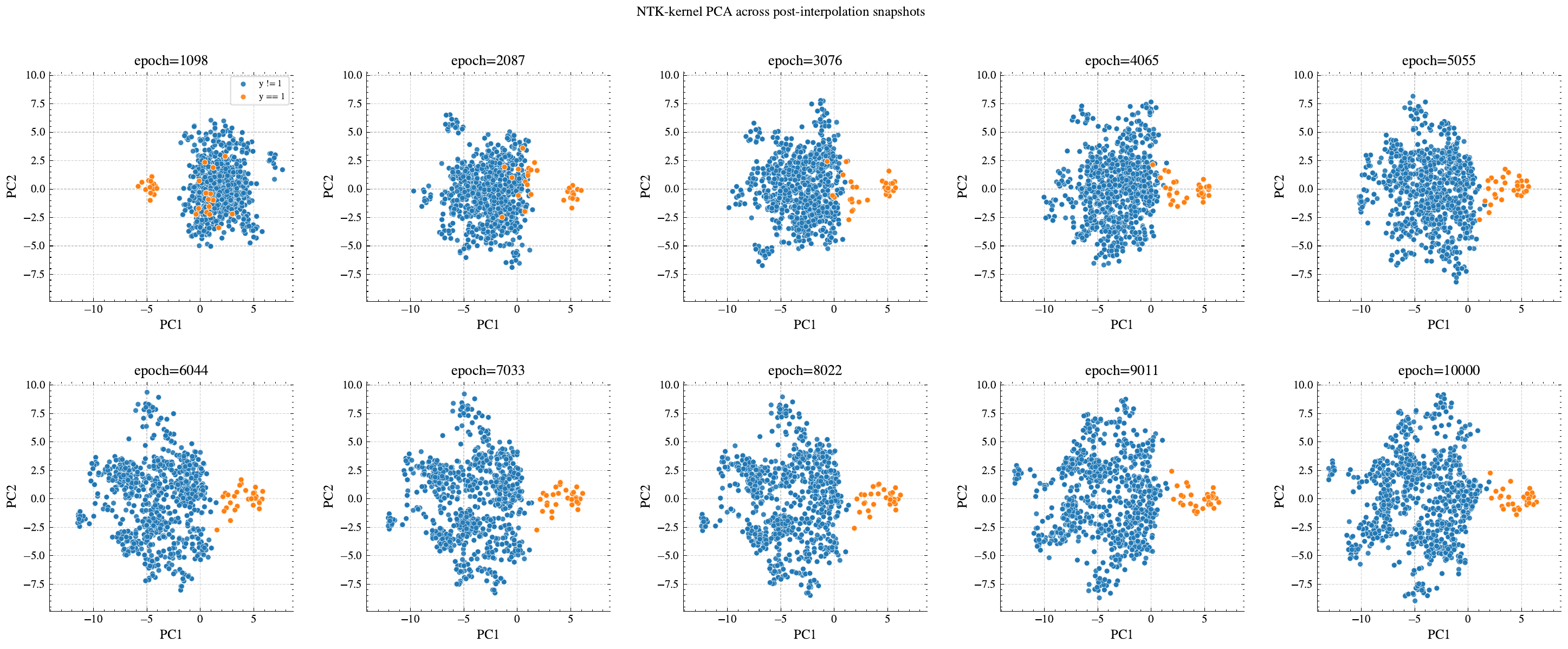}
    \caption{\textbf{Representation evolution during post-interpolation optimization.}
    We visualize NTK-based representations computed from the first output logit at different stages of the \texttt{GROKtimizer} trajectory. Immediately after interpolation, train and test samples from different classes remain partially overlapping, indicating that fitting the training data has not yet produced a representation that separates the test distribution. During the norm-minimization phase, the two classes progressively separate, and the representation develops a richer multi-cluster geometry. By the end of optimization, the classes are clearly separated, suggesting that post-interpolation dynamics continue to reorganize the learned representation even after the training set has been fit.}
    \label{fig:ntk_representation_dynamics}
\end{figure}

Figure~\ref{fig:ntk_representation_dynamics} shows that interpolation does not coincide with the end of meaningful representation learning. Immediately after the interpolation threshold, samples from different classes remain substantially mixed in the NTK representation space, especially among test points. This indicates that although the model has fit the training labels, the induced representation has not yet organized the input space in a way that supports robust generalization.

As the second phase progresses, the class-conditional point clouds separate more clearly. The geometry also becomes more structured: rather than remaining approximately isotropic, the representation develops multiple coherent clusters, suggesting that the optimizer continues to refine the internal organization of the data during complexity minimization. By the end of training, the two classes are well separated in the NTK representation, consistent with the observed improvement in test performance. These dynamics support the view that \texttt{GROKtimizer} does not merely reduce parameter norm after interpolation, but can also induce a qualitative reorganization of the learned representation in the post-interpolation regime.

\newpage

\subsection{WeightWatcher analysis}
\begin{wrapfigure}{r}{0.50\textwidth}
\vspace{-10pt}
    \centering
    \includegraphics[width=\linewidth]{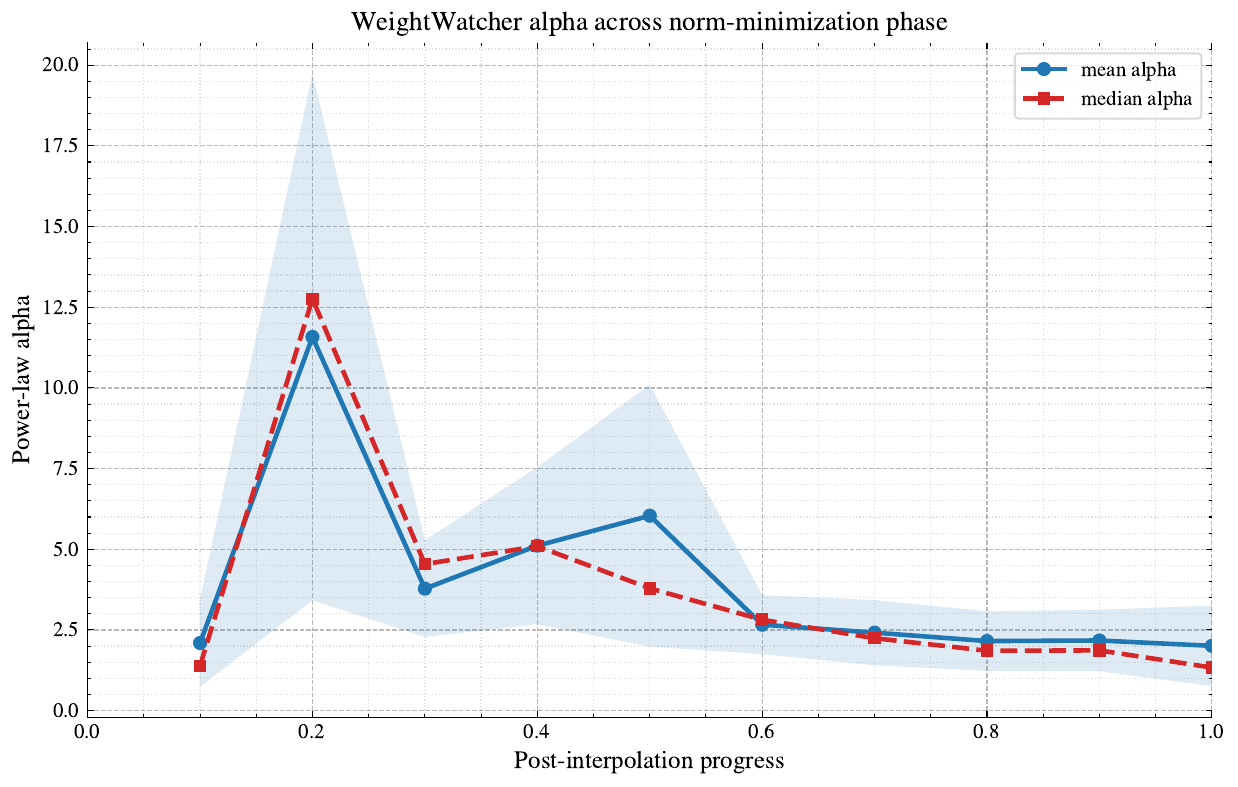}
    \caption{\textbf{WeightWatcher analysis during post-interpolation optimization.}
    Evolution of the fitted \texttt{WeightWatcher} exponent $\alpha$ on Modular Addition. During the complexity-minimization phase of \texttt{GROKtimizer}, $\alpha$ moves toward $2$, consistent with the emergence of a more strongly regularized spectral structure and with the \texttt{SETOL}
    \cite{martin2025setol} theory of generalization quality.}
    \label{fig:weightwatcher_modular}
\vspace{-10pt}
\end{wrapfigure}

We further examine the post-interpolation dynamics using the \texttt{WeightWatcher} framework, which characterizes the spectral structure of neural-network weight matrices through heavy-tailed power-law fits. In particular, we track the evolution of the fitted exponent $\alpha$ during training on Modular Addition. Prior work has associated smaller values of $\alpha$, and in particular values approaching $\alpha \approx 2$, with stronger implicit regularization and improved generalization quality~\cite{prakash2025grokking}.

Figure~\ref{fig:weightwatcher_modular} shows that the spectral structure of the model continues to change after interpolation. During the second phase of \texttt{GROKtimizer}, the fitted exponent $\alpha$ decreases toward $2$, suggesting that norm minimization is accompanied by a qualitative change in the weight spectra. This provides an additional diagnostic of the post-interpolation regime: the model is not only moving to lower parameter norm, but also toward a spectral profile previously associated with better generalization in grokking settings. We interpret this result as complementary evidence that the second phase of \texttt{GROKtimizer} induces a structured complexity-reduction process rather than merely continuing standard loss minimization.

\subsection{MNIST sample efficiency}

\begin{wrapfigure}{r}{0.50\textwidth}
\vspace{-10pt}
    \centering
    \includegraphics[width=\linewidth]{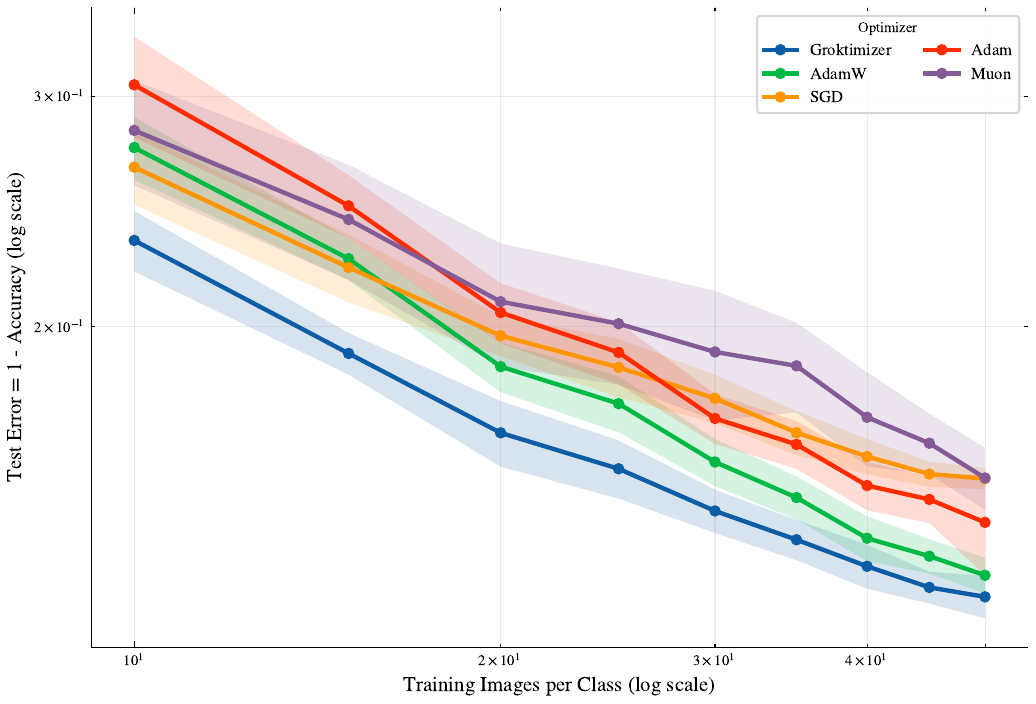}
    \caption{\textbf{MNIST sample-efficiency scaling.}
    Log--log plot of test error, defined as $1-\text{test accuracy}$, against the number of training examples.}
    \label{fig:mnist_loglog_error}
\vspace{-20pt}
\end{wrapfigure}

We additionally analyze sample efficiency on MNIST using a log--log scaling plot of test error, defined as $1-\text{test accuracy}$, against the number of training examples. This view separates two effects: the slope captures how quickly error decreases as additional data are added, while the vertical offset captures the absolute efficiency of each optimizer at a fixed sample size.

Figure~\ref{fig:mnist_loglog_error} shows that all optimizers exhibit broadly similar slopes, indicating comparable scaling behavior as the number of training examples increases. The main difference is therefore not the rate at which performance improves with more data, but the offset of the scaling curve. \texttt{GROKtimizer} achieves a consistently lower test-error offset across the full range of training-set sizes, meaning that it reaches lower error for the same number of examples. This is consistent with the interpretation that the biphasic schedule improves the selected interpolating solution rather than changing the fundamental data-scaling exponent.

\section{Extended Theoretical consideration} 

\subsection{Detailed proof of Lemma 1}

We now provide the proof of Lemma 1, which offers a simplified model of the importance of norm minimization for interpolating solutions of regression problems.

\textbf{Full statement of  Lemma 1}

Let $\mathbf X \in \mathbb R^{n \times p}$ be a Gaussian random matrix (with $p>n$), whose rows are sampled from an isotropic normal distribution $\boldsymbol{x}_i \sim \mathcal N(\mathbf 0,\mathbf I) \ \forall i$ . Let $\boldsymbol{w}_* \in \mathbb R^p$ be an unknown teacher vector and let  $\boldsymbol{y} =\mathbf X \boldsymbol{w}_* + \sigma \boldsymbol{\varepsilon}$ be a target vector, with $\sigma > 0$ and $\boldsymbol{\varepsilon} \sim \mathcal N(\mathbf 0, \mathbf I)$. Define $\hat {\boldsymbol{w}}_* = \mathbf X^\top (\mathbf X \mathbf X^\top)^{-1} \boldsymbol{y}$ as the \emph{minimum norm interpolating solution} and define the family of interpolating solutions
\begin{equation}
    \mathcal W_{\mathbf X,\boldsymbol{y}}
    = \left \{ \hat {\boldsymbol{w}}_* + (\mathbf I - 
    \mathbf X^\top (\mathbf X \mathbf X^\top)^{-1} \mathbf X
    ) \boldsymbol{z} \right\}_{\boldsymbol{z} \in \mathbb R^p}
\end{equation}
for $\boldsymbol{z} \in \mathbb R^p$. Define thus $\widehat {\boldsymbol{w}}$ as a random vector in $\mathcal W$, characterized by a perturbation $\boldsymbol{z} \sim \mathcal U(\mathcal S_r(0))$ where $r > 0$ represents the fixed norm of the random variable $\boldsymbol{z}$ in parameter space. Under the limit $n,p \to \infty$ with $n/p = \gamma \in (0,1)$ the expected generalization error has the following asymptotic behavior
\begin{equation}
    \mathbb E_{\boldsymbol{z},\boldsymbol{x},\mathbf X,\boldsymbol{\varepsilon} }
    (\boldsymbol{x}^\top \widehat{\boldsymbol{w}} - \boldsymbol{x}^\top \boldsymbol{w}_* )^2
    \sim (1 - \gamma) \| \boldsymbol{w}_*\|^2 + (1 - \gamma) r^2
    + \sigma^2 \frac{\gamma}{1 - \gamma}
\end{equation}
and since $ \mathbb E[\|\widehat {\boldsymbol w}_* - \widehat {\boldsymbol{w}}\|^2] = (1 - \gamma)r^2 $ the expected generalization error follows an asymptotic quadratic scaling with the distance from the minimum norm solution.
\\[1em]
\noindent
\emph{Proof}
\begin{align*}
     & 
     \mathbb E_{\boldsymbol{z}}
     \mathbb E_{\boldsymbol{x},\boldsymbol \varepsilon, \mathbf X} (\boldsymbol{x}^\top \widehat{\boldsymbol{w}} - \boldsymbol{x}^\top \boldsymbol{w}_* )^2
     \\ 
     & =
       \mathbb E_{\boldsymbol{z}}
     \mathbb E_{\boldsymbol \varepsilon, \mathbf X}
     \left[   (\widehat{\boldsymbol{w}}- \boldsymbol{w}_* )^\top \mathbb E_{\boldsymbol{x}}  ( \boldsymbol{x}\boldsymbol{x}^\top ) 
     (\widehat{\boldsymbol{w}} - \boldsymbol{w}_*)
     \right]
     \\
     & =
       \mathbb E_{\boldsymbol{z}}
     \mathbb E_{\boldsymbol \varepsilon, \mathbf X} 
     \| \widehat{\boldsymbol{w}} - \boldsymbol{w}_* \|^2_2
     \intertext{
     Since the minimum norm solution is defined by
     $\boldsymbol{w} = \lim_{\lambda \to 0}
    ( \mathbf X^\top \mathbf X + \lambda \mathbf I)^{-1} \mathbf X^\top \boldsymbol{y}
     = \mathbf X^\top (\mathbf X \mathbf X^\top)^{-1} \boldsymbol{y}
     $ we can rewrite $\hat {\boldsymbol{w}}$ explicitly
     } 
     & = 
       \mathbb E_{\boldsymbol{z}}
      \mathbb E_{\boldsymbol \varepsilon, \mathbf X} 
     \|
      \mathbf X^\top (\mathbf X \mathbf X^\top)^{-1} \boldsymbol{y}
       + (\mathbf I - 
    \mathbf X^\top (\mathbf X \mathbf X^\top)^{-1} \mathbf X
    ) \boldsymbol{z} - \boldsymbol{w}_*\|^2_2
     \\
     & = 
       \mathbb E_{\boldsymbol{z}}
      \mathbb E_{\boldsymbol \varepsilon, \mathbf X} 
     \|
      \mathbf X^\top (\mathbf X \mathbf X^\top)^{-1} (\mathbf X \boldsymbol{w}_* + \sigma \boldsymbol{\varepsilon})
      + (\mathbf I - 
    \mathbf X^\top (\mathbf X \mathbf X^\top)^{-1} \mathbf X
    ) \boldsymbol{z} - \boldsymbol{w}_*\|^2_2
     \intertext{
     Applying SVD decomposition $\mathbf X = \mathbf U \mathbf \Sigma \mathbf V^\top$ we rewrite
     }
     & = 
       \mathbb E_{\boldsymbol{z}}
      \mathbb E_{\boldsymbol \varepsilon, \mathbf X} 
     \|
      \mathbf V \mathbf V^\top \boldsymbol{w}_* + \sigma  \mathbf V \mathbf \Sigma^{-1} \mathbf U^\top \boldsymbol{\varepsilon}
      + (\mathbf I - \mathbf V \mathbf V^\top)\boldsymbol{z} - \boldsymbol{w}_*\|^2_2
     \\
     & = 
       \mathbb E_{\boldsymbol{z}}
      \mathbb E_{\boldsymbol \varepsilon, \mathbf X} 
     \|
      (\mathbf V \mathbf V^\top - \mathbf I) (\boldsymbol{w}_* - \boldsymbol{z}) + \sigma  \mathbf V \mathbf \Sigma^{-1} \mathbf U^\top \boldsymbol{\varepsilon}
      \|^2_2
     \\
     & = 
       \mathbb E_{\boldsymbol{z}}
      \mathbb E_{ \mathbf X} 
     \|
      (\mathbf V \mathbf V^\top - \mathbf I) (\boldsymbol{w}_* - \boldsymbol{z}) 
      \|^2_2
      +
     \sigma ^2
      \mathbb E_{\boldsymbol \varepsilon, \mathbf X} 
      \|
      \mathbf V \mathbf \Sigma^{-1} \mathbf U^\top \boldsymbol{\varepsilon}
      \|^2_2 & \text{since $\mathbb E[\boldsymbol{\varepsilon}] = \mathbf 0 $}
    \\
    & = 
      \mathbb E_{\boldsymbol{z}}
    \mathbb E_{\mathbf X}
    \sum_{i : \lambda_i = 0} (\boldsymbol{v}_i^\top (\boldsymbol{w}_* - \boldsymbol{z}) )^2
    + \sigma^2 \mathbb E_{\mathbf X} 
     \mathsf{Tr}
        \Big[ 
    \mathbf U \mathbf \Sigma^{-2} \mathbf U^\top
    \underbrace{ \mathbb E_{\boldsymbol{\varepsilon}}
[\boldsymbol{\varepsilon}\boldsymbol{\varepsilon}^\top]}_{\mathbf I}\Big]
\intertext{
Since the row space of a Gaussian random matrix is rotationally invariant, the nullspace projector acts isotropically in expectation, enabling us in using the uniform approximation of eigenvectors. Furthermore, noting that
$$
\mathsf{Tr}
        \left[ 
    \mathbf U \mathbf \Sigma^{-2} \mathbf U^\top\right] 
=
\frac{n}{p}
\frac{1}{n}
\mathsf{Tr}
        \left[ \left( \frac{1}{p} \mathbf X \mathbf X^\top \right)^{-1} \right] 
\to \frac{n}{p} \mathbb E_{s \sim \mathsf{MP}(\gamma)}[ s^{-1}] = \frac{n}{p} \frac{1}{1 - \gamma} = \frac{ \gamma }{1 - \gamma}
$$
we can apply the definition of first negative moment of the Marchenko Pastur distribution and obtain
}
& = \frac{p - n}{p}   \mathbb E_{\boldsymbol{z}} \| \boldsymbol{w}_* -\boldsymbol{z} \|^2  + \sigma^2 \frac{\gamma}{1 - \gamma}
\\
& = (1 - \gamma)   \mathbb E_{\boldsymbol{z}} \| \boldsymbol{w}_* -\boldsymbol{z} \|^2  + \sigma^2 \frac{\gamma}{1 - \gamma}
\intertext{Finally, taking the expectation with respect $\boldsymbol{z}$ we obtain}
& = (1 - \gamma) \| \boldsymbol{w}_*\|^2 + (1 - \gamma) r^2 + \sigma^2\frac{\gamma}{1 - \gamma}
\intertext{
and noting that
\begin{equation}
    \mathbb E \| \widehat{\boldsymbol{w}}
     - \widehat{\boldsymbol{w}}_* \|^2 
     = 
     \mathbb E
     \left[
     \|
     (\mathbf I - 
    \mathbf X^\top (\mathbf X \mathbf X^\top)^{-1} \mathbf X
    ) \boldsymbol{z}
     \|^2
     \right]
     =  \frac{r^2}{p} \left(
     p - p \gamma 
     \right)
     = r^2(1 - \gamma)
\end{equation}
we obtain the final formula
}
   & =  \mathbb E_{\boldsymbol{z},\boldsymbol{x},\mathbf X,\boldsymbol{\varepsilon} }
    (\boldsymbol{x}^\top \widehat{\boldsymbol{w}} - \boldsymbol{x}^\top \boldsymbol{w}_* )^2
    = 
    (1 - \gamma)\|\boldsymbol{w}_*\|^2
    +
    \mathbb 
    E_{\boldsymbol{z},\boldsymbol{x},\mathbf X,\boldsymbol{\varepsilon} }\| \widehat{\boldsymbol{w}}
     - \widehat{\boldsymbol{w}}_* \|^2 
    + \sigma^2 \frac{\gamma}{1 - \gamma}
\end{align*}
\subsection{Detailed proof of Theorem 2}

It is known that for a momentum based optimizer with momentum rate $\beta$ on a loss function of Hessian such that its largest eigenvalue equals $L$,  the learning rate has to belong to the positive interval

$$
\eta \in \left( 0 , \frac{2(1 + \beta)}{L} \right)
$$

Under the approximation of local convexity of the regularized loss landscape
$$
f(\boldsymbol{w}) = \frac{1}{2}(\boldsymbol{w} -\boldsymbol{w}_\star)^\top \mathbf H (\boldsymbol{w} - \boldsymbol{w}_\star) +  \frac{1}{2} \lambda \| \boldsymbol{w} \|^2
$$
the Hessian of the function is
\begin{equation}
    \mathcal H = \mathbf H + \lambda \mathbf I
\end{equation}
and thus
\begin{equation}
    L = \lambda_{\max} + \lambda
\end{equation}
we thus obtain
\begin{equation}
    \eta_{\max} = \frac{2(1 + \beta)}{\lambda_{\max} + \lambda}
\end{equation}
given the condition of critical damping which is
\begin{equation}
    \beta = 1 - 2 \sqrt{ \eta \lambda}
\end{equation}
we obtain the nonlinear system
\begin{equation}
    \begin{cases}
    \eta_{\max} &= \frac{2(1 + \beta)}{\lambda_{\max} + \lambda}
    \\
        \beta & = 1 - 2 \sqrt{\eta_{\max} \lambda}
    \end{cases}
\end{equation}
which is solved, under the constraint $\eta_{\max} > 0$, for
\begin{equation}
    \eta_{\max} = 
    \frac{
4 \left( 
\sqrt{\lambda_{\max} + 2 \lambda} - \sqrt{\lambda}
\right)^2
    }{(\lambda_{\max} + \lambda)^2}
    \sim 4 \lambda_{\max}^{- 1}
\end{equation}
which represents the maximally admissable learning rate to guarantee convergence.

\subsection{The Critically Damped Regime}

We provide, now, the complete proof of the quadratic speed up of the Critically Damped Momentum optimizer.

We start by constructing the continuous limit of the momentum optimizer.
Consider a Heavy Ball iteration of the form
\begin{equation}
    \begin{cases}
        \boldsymbol{m}_{k + 1} = \beta \boldsymbol{m}_k - \eta \nabla \mathcal L(\boldsymbol{w}_k) \\
        \boldsymbol{w}_{k + 1} = \boldsymbol{w}_k + \boldsymbol{m}_{k + 1}
    \end{cases}.
\end{equation}
\noindent
Observe that
\begin{equation}
    \boldsymbol{m}_{k + 1} = \boldsymbol{w}_{k + 1} - \boldsymbol{w}_k
\end{equation}
therefore, we can rewrite the iteration as
\begin{equation}
        \boldsymbol{w}_{k + 1} - \boldsymbol{w}_k = \beta \left( 
        \boldsymbol{w}_{k} - \boldsymbol{w}_{k - 1}
        \right) - \eta \nabla \mathcal L(\boldsymbol{w}_k) 
\end{equation}
adding $\boldsymbol w_{k - 1} - \boldsymbol{w}_{k} $ to both terms we obtain
\begin{equation}
        \boldsymbol{w}_{k + 1} - 2 \boldsymbol{w}_k  + \boldsymbol{w}_{k-1}= ( \beta - 1) \left( 
        \boldsymbol{w}_{k} - \boldsymbol{w}_{k - 1}
        \right) - \eta \nabla \mathcal L(\boldsymbol{w}_k) 
\end{equation}
assume now that
\begin{equation}
\begin{cases}
    1 - \beta  := \gamma := \gamma^u \sqrt{ \eta }
\end{cases}
\end{equation}
we can rewrite
\begin{equation}
    \boldsymbol{w}_{k + 1} - 2 \boldsymbol{w}_k  + \boldsymbol{w}_{k-1}= - \sqrt \eta \gamma^u \left( 
        \boldsymbol{w}_{k} - \boldsymbol{w}_{k - 1}
        \right) - \eta \nabla \mathcal L(\boldsymbol{w}_k) 
\end{equation}
dividing both terms by $\eta$ (which acts a time scale parameter) we obtain
\begin{equation}
    \frac{\boldsymbol{w}_{k + 1} - 2 \boldsymbol{w}_k  + \boldsymbol{w}_{k-1}}{(\sqrt \eta)^2} = -  \gamma^u \frac{\boldsymbol{w}_{k} - \boldsymbol{w}_{k - 1}}{\sqrt \eta } - \nabla \mathcal L(\boldsymbol{w}_k)
\end{equation}
which in the small $\sqrt \eta$ limit becomes
\begin{equation}
    \ddot{\boldsymbol{w}} + \gamma^u \dot{\boldsymbol{w}} +   \nabla \mathcal L(\boldsymbol{w}) = \boldsymbol{0}
\end{equation}
which provides a continuous dynamics of the momentum optimizer.
The differential equation, written in this form, is a challenging object to study due to the gradient of the loss. We then proceed to provide a change of coordinate system that is mathematically tractable.
Consider the eigenbasis $\{(\lambda_i, \boldsymbol{v}_i)\}_{i=1}^N$ of $\mathbf H$ such that
\begin{equation}
    \mathbf H = \sum_i \lambda_i \boldsymbol{v}_i \boldsymbol{v}_i^\top
\end{equation}
we can rewrite the dynamics equation as
\begin{equation}
   \boldsymbol{v}_i^\top \ddot {\boldsymbol{w}} +
   \gamma^u \boldsymbol{v}_i^\top \dot {\boldsymbol{w}} +\boldsymbol{v}_i^\top \nabla \mathcal L(\boldsymbol{w}) = 0 \ \ \ \forall i
\end{equation}
explicitating the structure of the gradient we obtain
\begin{equation}
\begin{split}
   \boldsymbol{v}_i^\top \ddot {\boldsymbol{w}} 
   + 
   \gamma^u \boldsymbol{v}_i^\top \dot {\boldsymbol{w}} +  \boldsymbol{v}_i^\top \left( \mathbf H (\boldsymbol{w} - \boldsymbol{w}^*) + \lambda \boldsymbol{w}\right)
   & = 0
   \\
      \boldsymbol{v}_i^\top \ddot {\boldsymbol{w}} 
      +
      \gamma^u \boldsymbol{v}_i^\top \dot {\boldsymbol{w}}
   +  \lambda_i \boldsymbol{v}_i^\top (\boldsymbol{w} - \boldsymbol{w}^*) +  \lambda \boldsymbol{v}_i^\top \boldsymbol{w}
   & = 0
   \\
   \boldsymbol{v}_i^\top \ddot {\boldsymbol{w}} +
    \gamma^u \boldsymbol{v}_i^\top \dot {\boldsymbol{w}}
   + (\lambda_i + \lambda) \boldsymbol{v}_i^\top \boldsymbol{w} -
    \lambda_i \boldsymbol{v}_i^\top \boldsymbol{w}^*
   & = 
0
\\
   \partial_t^2 \left \{ \boldsymbol{v}_i^\top  {\boldsymbol{w}} \right\} +
    \gamma^u \partial_t \left\{ \boldsymbol{v}_i^\top {\boldsymbol{w}} \right\}
   +  (\lambda_i + \lambda) \boldsymbol{v}_i^\top \boldsymbol{w} -
    \lambda_i \boldsymbol{v}_i^\top \boldsymbol{w}^*
   & = 
0
\end{split}
\end{equation}
where $\boldsymbol{w}^*$ is an arbitrary point such that
\begin{equation}
    \mathbf H(\boldsymbol{w}^\star - \boldsymbol{w}^*) = 0
\end{equation}
For clarity, let us denote as 
\begin{equation}
    \tilde w_i = \boldsymbol{v}_i^\top \boldsymbol{w},  \ \ \ \tilde w^*_i = \boldsymbol{v}_i^\top \boldsymbol{w}^*
\end{equation}
which leads to the system of decoupled ODEs
\begin{equation}
    \partial_t^2 \tilde w_i + \gamma^u \partial_t \tilde w_i + 
    (\lambda_i + \lambda) \tilde w_i -  \lambda_i \tilde w^*_i = 0 \ \ \ \forall i
\end{equation}
Since we start from interpolation regime, we focus only components associated to the null space of the hessian, that is

\begin{equation}
       \partial_t^2 \tilde w_i + \gamma^u \partial_t \tilde w_i + 
     \lambda \tilde w_i = 0 \ \ \ \forall i : \lambda_i = 0 
\end{equation}
and we search the optimal $\gamma^u$ to maximize the convergence speed. 
Recognizing the harmonic oscillator equation, the fastest (non oscillating solution) is the \textbf{critically damped solution} defined as

\begin{equation}
    \tilde w^{CD}_i(t) = 
    \tilde w^{CD}_i(0) \left(1 + {\gamma^u}{t} \right) e^{ - \frac{\gamma^u}{2}t} 
\end{equation}
which is obtained for
\begin{equation}
    \gamma^u = 2 \sqrt{\lambda}
\end{equation}
where we remember, $\lambda$ is the weight decay factor. This gives the decay rate
\begin{equation}
     \tilde w_i^{CD}(t) \sim  e^{ - \sqrt{ \lambda } t } 
\end{equation}
against the $\sim e^{- 2 \lambda t}$ of classic gradient descent.
Replacing in $\beta$ we obtain
\begin{align*}
    1 -\beta & = \gamma^u \sqrt \eta  \\
    1 - \beta & = 2 \sqrt{ \lambda \eta} \\
    \beta & = 1 - 2 \sqrt{ \lambda \eta }
\end{align*}
which is the optimal $\beta$ parameter.
Let us study the impact on other coordinates. Consider again the again for each eigendirection $i$, fixing this time $\gamma^u = 2 \sqrt{ \lambda}$
\begin{equation}
    \partial_t^2 \tilde w_i + 2 \sqrt{ \lambda} \partial_t \tilde w_i + 
    (\lambda_i + \lambda) \tilde w_i -  \lambda_i \tilde w^*_i = 0 \ \ \ \forall i \neq 0
\end{equation}
and assume $\tilde w_i(0) = \tilde w^*_i$ and $\partial_t \tilde w_i(0) = 0$. 
\\ 
\\
To solve the given second-order linear ordinary differential equation (ODE) for $\tilde w_i(t)$, we aknowledge $\tilde w^*_i$, $\lambda$, and $\lambda_i$ as constants. We assume that $\lambda_i \gg \lambda$, meaning that "informative" directions cause stronger curvature than the regularization.

The equation, in canonical form, is:
\begin{equation}
    \partial_t^2 \tilde w_i + 2 \sqrt{ \lambda} \partial_t \tilde w_i + (\lambda_i + \lambda) \tilde w_i =  \lambda_i \tilde w^*_i
\end{equation}

We first find the steady-state or particular solution, $\tilde w_{i, p}$, by setting all time derivatives to zero:
\begin{align*}
    (\lambda_i + \lambda) \tilde w_{i, p} &= \lambda_i \tilde w^*_i \\[6pt]
    \tilde w_{i, p} &= \frac{\lambda_i}{\lambda_i + \lambda} \tilde w^*_i
\end{align*}

Next, we solve the homogeneous equation:
\begin{equation}
    \partial_t^2 \tilde w_i + 2 \sqrt{ \lambda} \partial_t \tilde w_i + (\lambda_i + \lambda) \tilde w_i = 0
\end{equation}

The characteristic equation for this ODE is:
\begin{equation}
    r^2 + 2 \sqrt{ \lambda} r + (\lambda_i + \lambda) = 0
\end{equation}

Using the quadratic formula to find the roots $r$:
\begin{align*}
    r &= \frac{- 2 \sqrt{ \lambda} \pm \sqrt{4 \lambda - 4(\lambda_i + \lambda)}}{2} \\[6pt]
    &= \frac{-2\sqrt{\lambda} \pm \sqrt{4\lambda - 4\lambda_i - 4\lambda}}{2} \\[6pt]
    &= \frac{-2\sqrt{ \lambda} \pm \sqrt{-4\lambda_i}}{2} \\[6pt]
    &= -\sqrt{ \lambda} \pm i \sqrt{ \lambda_i}
\end{align*}

Since the roots are complex (assuming $\lambda > 0$ and $\lambda_i > 0$), the solution to the homogeneous equation takes the form of an exponentially decaying sine and cosine wave:
\begin{equation}
    \tilde w_{i, h}(t) = e^{-\sqrt{\lambda} t} \left( A \cos(\sqrt{\lambda_i} t) + B \sin(\sqrt{\lambda_i} t) \right)
\end{equation}

The general solution is the sum of the homogeneous and particular solutions:
\begin{equation}
    \tilde w_i(t) = e^{-\sqrt{\lambda} t} \left( A \cos(\sqrt{\lambda_i} t) + B \sin(\sqrt{\lambda_i} t) \right) + \frac{\lambda_i}{\lambda_i + \lambda} \tilde w^*_i
\end{equation}

We are given two initial conditions to solve for constants $A$ and $B$.

Substitute $t = 0$ into the general solution:
\begin{align*}
    \tilde w^*_i &= e^0 (A \cdot 1 + B \cdot 0) + \frac{\lambda_i}{\lambda_i + \lambda} \tilde w^*_i \\[6pt]
    A &= \tilde w^*_i - \frac{\lambda_i}{\lambda_i + \lambda} \tilde w^*_i \\[6pt]
    A &= \frac{\lambda}{\lambda_i + \lambda} \tilde w^*_i
\end{align*}

First, take the derivative of the general solution $\tilde w_i(t)$ with respect to $t$:
\begin{align*}
    \partial_t \tilde w_i(t) &= -\sqrt{\lambda} e^{-\sqrt{\lambda} t} \big( A \cos(\sqrt{\lambda_i} t) + B \sin(\sqrt{\lambda_i} t) \big) \\
    &\quad + e^{-\sqrt{\lambda} t} \big( -A \sqrt{\lambda_i} \sin(\sqrt{\lambda_i} t) + B \sqrt{\lambda_i} \cos(\sqrt{\lambda_i} t) \big)
\end{align*}

Substitute $t = 0$ and set it equal to $0$:
\begin{align*}
    0 &= -\sqrt{\lambda} (A) + (B \sqrt{\lambda_i}) \\[6pt]
    B \sqrt{\lambda_i} &= A \sqrt{\lambda} \\[6pt]
    B &= A \sqrt{\frac{\lambda}{\lambda_i}}
\end{align*}

Now, plug in our result for $A$:
\begin{equation}
    B = \left( \frac{\lambda}{\lambda_i + \lambda} \tilde w^*_i \right) \sqrt{\frac{\lambda}{\lambda_i}} = \frac{\lambda \sqrt{\lambda}}{\sqrt{\lambda_i}(\lambda_i + \lambda)} \tilde w^*_i
\end{equation}

Substitute the constants $A$ and $B$ back into the general solution:
\begin{equation}
    \tilde w_i(t) = \tilde w^*_i \left[ \frac{\lambda_i}{\lambda_i + \lambda} + e^{-\sqrt{\lambda} t} \left( \frac{\lambda}{\lambda_i + \lambda} \cos(\sqrt{\lambda_i} t) + \frac{\lambda \sqrt{\lambda}}{\sqrt{\lambda_i}(\lambda_i + \lambda)} \sin(\sqrt{\lambda_i} t) \right) \right]
\end{equation}

In the Regime $\lambda_i \gg \lambda$ we have
\begin{equation}
    \frac{\lambda}{\lambda_i + \lambda} \sim \frac{\lambda}{\lambda_i}
\end{equation}
and
\begin{equation}
    \frac{\lambda \sqrt{\lambda}}{\sqrt{\lambda_i}(\lambda_i + \lambda)}
    \sim \left( \frac{\lambda}{\lambda_i} \right)^{3/2}
\end{equation}
and since $\lambda_i \gg \lambda$ we get
\begin{equation}
    \left( \frac{\lambda}{\lambda_i} \right)^{3/2} \ll \frac{\lambda}{\lambda_i}.
\end{equation}
We notice that the dynamics is characterized by small $\mathcal O (\lambda{\lambda_i}^{-1})$ fluctuations --- fast oscillating since their frequency is $\mathcal O( \lambda_i^{1/2})$ ---  that decay at the same rate as the null space components, but do not affect the true directions due to their small magnitude.
Therefore, the convergence rate of the momentum algorithm in this setting is
\begin{equation}
    \mathcal L_{reg}(t) \sim e^{ - \sqrt{ \lambda } t}
\end{equation}
which gives us the time scale
\begin{equation}
    \tau_{CDM} \sim \frac{1}{\sqrt{\lambda}} = \sqrt{\tau_{GD}}
\end{equation}
which employing the equivalence $\tau \sim E \Delta t$ (where $\Delta t$ is the discretization step in the continuous limit) leads to the relationship
\begin{equation}
    E_{CDM} \sim \sqrt{E_{GD}}
\end{equation}
proving our thesis.

\section{Experimental details}
\subsection{Computational costs}

All experiments were conducted on a
  workstation with dual AMD EPYC 7713
  CPUs (2 × 64 cores; 128 physical
  cores total) and an NVIDIA RTX 6000
  Ada Generation GPU (49,140 MiB VRAM,
  ~49 GB). Wall-clock times for the synthetic benchmarks and the real-world tabular datasets are provided in Table \ref{tab:appendix_benchmark_results_transposed}.

\begin{table}[!htbp]
\centering
\caption{\textbf{Wall clock time (s) across datasets.}
Results are reported as mean $\pm$ uncertainty across runs. Values are rounded to a precision appropriate for the reported uncertainty. Best values for each dataset are shown in bold.}
\label{tab:appendix_benchmark_results_transposed}
\scriptsize
\setlength{\tabcolsep}{4.5pt}
\renewcommand{\arraystretch}{1.15}
\begin{tabular}{lccccc}
\toprule
\textbf{Dataset}
& \textbf{Adam}
& \textbf{AdamW}
& \textbf{\texttt{GROKtimizer}}
& \textbf{Muon}
& \textbf{SGD} \\
\midrule
Binary Addition
& $456.5 \pm 29.6$
& $461.1 \pm 28.2$
& $\mathbf{453.7 \pm 32.1}$
& $455.9 \pm 30.0$
& $460.3 \pm 28.5$ \\

Gaussian
& $143.9 \pm 15.5$
& $143.6 \pm 15.6$
& $210.8 \pm 23.5$
& $143.0 \pm 15.7$
& $\mathbf{142.7 \pm 15.6}$ \\

Leukemia
& $0.473 \pm 0.009$
& $0.472 \pm 0.010$
& $0.405 \pm 0.004$
& $0.816 \pm 0.024$
& $\mathbf{0.377 \pm 0.007}$ \\

Modular Addition
& $1969 \pm 131$
& $1963 \pm 139$
& $\mathbf{1954 \pm 144}$
& $1956 \pm 139$
& $1972 \pm 137$ \\

QM9
& $3557 \pm 197$
& $3558 \pm 201$
& $\mathbf{3536 \pm 215}$
& $3539 \pm 212$
& $3539 \pm 208$ \\

RF Teacher Linear
& $5.63 \pm 0.24$
& $5.56 \pm 0.18$
& $5.28 \pm 0.28$
& $5.69 \pm 0.26$
& $\mathbf{4.74 \pm 0.18}$ \\

Sparse Parity
& $\mathbf{291.8 \pm 25.5}$
& $308.0 \pm 24.1$
& $297.7 \pm 24.7$
& $321.7 \pm 25.6$
& $305.5 \pm 24.5$ \\

TCGA
& $1.94 \pm 0.12$
& $1.90 \pm 0.10$
& $1.82 \pm 0.10$
& $2.45 \pm 0.09$
& $\mathbf{1.76 \pm 0.08}$ \\

Two-Subspace Linear
& $4.449 \pm 0.008$
& $4.470 \pm 0.016$
& $3.950 \pm 0.008$
& $4.398 \pm 0.007$
& $\mathbf{3.658 \pm 0.015}$ \\
\bottomrule
\end{tabular}
\end{table}

  \subsection{Dataset details}

  \textbf{Modular addition.}
  A synthetic algorithmic classification task
  where inputs are integer pairs $(a,b)$ and
  the label is $(a+b)\bmod 31$. In this repo,
  inputs are one-hot encoded (concatenated
  pair encoding), ordered pairs are included,
  and the default split is 50/50 train/test
  over the pair pool.
  Model: 3-layer MLP with hidden sizes
  $[1024,1024,1024]$.
  Hyperparameters: $\text{lr}=10^{-3}$, $
  \text{wd}=10^{-5}$, epochs $=10{,}000$.

  \textbf{Sparse parity.}
  A synthetic binary classification task built
  from random binary vectors of dimension 20;
  the target is parity over a selected support
  (here parity over 20 bits). Features are
  represented as $\{-1,+1\}$ values in the
  generator. Default sizes are 10,000 train
  and 1,000 test examples.
  Model: 3-layer MLP with hidden sizes
  $[2048,2048,2048]$, binary output.
  Hyperparameters: $\text{lr}=10^{-3}$, $
  \text{wd}=10^{-3}$, epochs $=1{,}000$.

  \textbf{Binary addition.}
  A synthetic supervised task where two 10-bit
  integers are concatenated as input and the
  11-bit binary sum is predicted as output.
  The generator creates random examples with
  1,500 training and 2,000 test samples.
  Model: 3-layer MLP with hidden sizes
  $[2048,2048,2048]$, output size 11.
  Hyperparameters: $\text{lr}=10^{-4}$, $
  \text{wd}=10^{-3}$, epochs $=2{,}000$.

  \textbf{TCGA.}
  A real gene-expression binary classification from KIRC \cite{cancer2013comprehensive}.
  Labels are derived from sample identifiers;
  split is stratified with a very small train
  fraction (test size 0.95).
  Model: linear binary classifier.
  Hyperparameters: $\text{lr}=10^{-3}$, $
  \text{wd}=10^{-3}$, epochs $=1{,}000$.

  \textbf{Leukemia.}
  A real biomedical binary classification
  dataset \cite{golub1999molecular} fetched from OpenML
  (\texttt{leukemia}, v1). Data are stratified
  into train/test (test size 0.8), and
  standardized using \texttt{StandardScaler}.
  Model: linear binary classifier.
  Hyperparameters: $\text{lr}=10^{-3}$, $
  \text{wd}=10^{-3}$, epochs $=1{,}000$.

  \textbf{QM9.}
  A molecular property regression task using
  PyTorch Geometric’s QM9 loader. In this
  config, a subset of size 100 is used, split
  50/50, and target index 1 is predicted (with
  optional target standardization).
  Model: GNN-style
  architecture (edge features + MLP head;
  hidden sizes 32/16 and FC hidden 2048).
  Hyperparameters: $\text{lr}=10^{-3}$, $
  \text{wd}=10^{-2}$, epochs $=5{,}000$.

  \textbf{MNIST.}
  Digit classification on MNIST with
  controllable low-sample training subsets,
  with full test set retained. Inputs are
  normalized and flattened vectors.
  Model: 3-layer MLP with hidden sizes
  $[1024,512,256]$, 10-way output.
  Hyperparameters: $\text{lr}=5\times10^{-3}$,
  $\text{wd}=10^{-3}$, epochs $=5{,}000$.

  \textbf{TwoSubspaceLinear.}
  Synthetic binary classification where labels
  depend only on a low-dimensional signal
  subspace, while extra junk dimensions are
  appended. Default config uses 384 train /
  1024 test, 16 signal dims, 224 junk dims.
  Model: linear classifier (logistic
  regression style).
  Hyperparameters: $\text{lr}=10^{-3}$, $
  \text{wd}=10^{-3}$, epochs $=10{,}000$.

  \textbf{RandomFeaturesTeacherLinear.}
  Synthetic teacher-student binary task: a
  random-feature teacher generates logits,
  then labels are thresholded and optionally
  noised. Default config uses 512 train / 1024
  test and 512 input features.
  Model: linear classifier.
  Hyperparameters: $\text{lr}=10^{-3}$, $
  \text{wd}=10^{-3}$, epochs $=10{,}000$.

  \textbf{Gaussian.}
  Synthetic linear regression where train
  labels are generated by a random linear map
  and test labels follow the interpolating
  solution. Default sizes are 100 train / 1000
  test with 200 features.
  Model: dot-product linear model (no bias).
  Hyperparameters: $\text{lr}=10^{-2}$, $
  \text{wd}=10^{-3}$, epochs $=3{,}000$.

  \textbf{WikiText-2.}
  Character-level language modeling from
  WikiText-2 raw text splits.
  Model: GPT-style causal LM with
  $n_{\text{layer}}=8$, $n_{\text{embd}}=384$,
  $n_{\text{head}}=6$, block size 256, batch
  size 192.
  Hyperparameters: $\text{lr}=3\times10^{-3}$,
  $\text{wd}=10^{-3}$, num\_steps $=8{,}000$.

  \textbf{BabyLM Strict Small.}
  Character-level language modeling over
  concatenated BabyLM strict-small sources,
  split by \texttt{train\_fraction=0.9}.
  Model: GPT-style causal LM with
  $n_{\text{layer}}=12$, $n_{\text{embd}}
  =512$, $n_{\text{head}}=8$, block size 256,
  batch size 128.
  Hyperparameters: $\text{lr}=2\times10^{-3}$,
  $\text{wd}=10^{-3}$, num\_steps $=12{,}000$.

\end{document}